\documentclass[11pt]{article}
\usepackage[a4paper,top=3cm,bottom=2cm,left=2cm,right=2cm,marginparwidth=1.75cm]{geometry}

\usepackage[utf8]{inputenc} 
\usepackage[T1]{fontenc}
\usepackage[numbers]{natbib}
\usepackage[colorlinks=True,citecolor=blue,urlcolor=blue,pagebackref=true,backref=true]
{hyperref}
\renewcommand*{\backrefalt}[4]{\ifcase #1 \footnotesize{(Not cited.)}\or        \footnotesize{(Cited on page~#2.)}\else      \footnotesize{(Cited on pages~#2.)}\fi}

\usepackage[none]{hyphenat}
\usepackage{breakcites}
\usepackage[utf8]{inputenc} \usepackage[T1]{fontenc}    \usepackage{hyperref}       \hypersetup{
  colorlinks = true,
  urlcolor = blue,
  linkcolor = blue,
  citecolor = blue
}
\usepackage{booktabs}       \usepackage{amsfonts}       \usepackage{nicefrac}       \usepackage{microtype}      \usepackage{xcolor}         \usepackage{caption}
\usepackage{subcaption}
\usepackage{float}
\usepackage{amsmath,amsthm}
\usepackage{amssymb}
\usepackage{mathtools}
\usepackage{natbib}
\usepackage{soul}
\usepackage{bm}
\usepackage{enumitem}
\usepackage{multirow}
\usepackage{wrapfig}
\usepackage{scalerel}
\allowdisplaybreaks
\usepackage{dsfont}
\usepackage[noend]{algpseudocode}
\usepackage{xcolor}
\usepackage{thmtools,thm-restate}
\usepackage{cleveref}

 \newtheorem{theorem}{Theorem}
 \newtheorem{definition}{Definition}
 \newtheorem{lemma}{Lemma}
 \newtheorem{remark}{Remark}
\newtheorem{proposition}{Proposition}

\title{Finite Sample Identification of \\ Partially Observed Bilinear Dynamical Systems}
\author{Yahya Sattar$^\star$ \\ Cornell \\ ysattar@cornell.edu \and  Yassir Jedra$^\star$ \\ MIT  \\ jedra@mit.edu \and  Maryam Fazel \\ U Washington \\ mfazel@uw.edu \and Sarah Dean \\ Cornell  \\ sdean@cornell.edu}

\date{}
\usepackage{enumitem}
\usepackage{caption,subcaption,multirow}
\usepackage{bm}

\renewcommand{\cite}{\citep}

\newtheorem{assumption}{Assumption}

\newtheorem{assumptionalt}{Assumption}[assumption]

\newenvironment{assumptionp}[1]{
  \renewcommand\theassumptionalt{#1}
  \assumptionalt
}{\endassumptionalt}
\newcommand{\poly}{\textup{poly}}

\newcommand{\bv}[1]{{\boldsymbol{#1}}}		\newcommand{\bvgrk}[1]{{\boldsymbol{#1}}}

\newcommand{\util}{\tilde{u}}

\newcommand{\utiltil}{\tilde{\util}}
\newcommand{\vutiltil}{\tilde{\vutil}}

\newcommand{\vb}{\bv{b}}

\newcommand{\vf}{\bv{f}}

\newcommand{\vq}{\bv{q}}
\newcommand{\vr}{\bv{r}}

\newcommand{\vu}{\bv{u}}
\newcommand{\vv}{\bv{v}}
\newcommand{\vw}{\bv{w}}
\newcommand{\vx}{\bv{x}}
\newcommand{\vy}{\bv{y}}
\newcommand{\vz}{\bv{z}}

\newcommand{\ub}{\bv{u}}
\newcommand{\wb}{\bv{w}}
\newcommand{\xb}{\bv{x}}

\newcommand{\yb}{\bv{y}}
\newcommand{\zb}{\bv{z}}

\newcommand{\vutil}{\tilde{\vu}}

\newcommand{\vwtil}{\tilde{\vw}}

\newcommand{\Ecal}{\mathcal{E}}

\newcommand{\Ncal}{\mathcal{N}}
\newcommand{\Ocal}{\mathcal{O}}

\newcommand{\Scal}{\mathcal{S}}

\newcommand{\tn}[1]{\|{#1}\|_{\ell_2}}
\newcommand{\vA}{\bv{A}}
\newcommand{\vB}{\bv{B}}
\newcommand{\vC}{\bv{C}}
\newcommand{\vD}{\bv{D}}

\newcommand{\vF}{\bv{F}}
\newcommand{\vG}{\bv{G}}

\newcommand{\vI}{\bv{I}}

\newcommand{\vM}{\bv{M}}
\newcommand{\vN}{\bv{N}}

\newcommand{\vQ}{\bv{Q}}
\newcommand{\vR}{\bv{R}}

\newcommand{\vV}{\bv{V}}
\newcommand{\vW}{\bv{W}}
\newcommand{\vX}{\bv{X}}

\newcommand{\Ab}{\bv{A}}

\newcommand{\Bb}{\bv{B}}
\newcommand{\Cb}{\bv{C}}
\newcommand{\Db}{\bv{D}}

\newcommand{\Fb}{\bv{F}}
\newcommand{\Gb}{\bv{G}}

\newcommand{\Mb}{\bv{M}}

\newcommand{\vGhat}{\hat{\bv{G}}}

\newcommand{\vUtil}{\tilde{\bv{U}}}

\newcommand{\veta}{\bvgrk{\eta}}

\newcommand{\vSigma}{\bvgrk{\Sigma}}

\newcommand{\nn}{\nonumber}

\newcommand{\mtx}[1]{\bm{#1}}

\newcommand{\E}{\operatorname{\mathbb{E}}}
\newcommand{\PP}{\operatorname{\mathbb{P}}}

\newcommand{\tf}[1]{\|{#1}\|_{F}}

\newcommand{\Iden}{{\mtx{I}}}

\renewcommand{\P}{\operatorname{\mathbb{P}}}
\newcommand{\leqsym}[1]{\stackrel{\text{(#1)}}{\leq}}

\newcommand{\eqsym}[1]{\stackrel{\text{(#1)}}{=}}

\newcommand{\norm}[1]{\|{#1} \|}

\newcommand{\R}{\mathbb{R}}				

\newcommand{\dm}[2]
{
	\IfStrEq{#2}{1}{\R^{#1}}{\R^{#1 \x #2}}
}

 			\newcommand{\tr}{\textup{\textbf{tr}}} 			 		 					\newcommand{\mat}{\textup{\textbf{mtx}}}			 	 	  				 				 											\newcommand{\distas}{\overset{\text{i.i.d.}}{\sim}}

\newcommand{\T}{\top}
\newcommand*{\x}{\mathsf{x}\mskip1mu}

\newcommand{\splitatcommas}[1]{\begingroup
	\begingroup\lccode`~=`, \lowercase{\endgroup
		\edef~{\mathchar\the\mathcode`, \penalty0 \noexpand\hspace{0pt plus .1em}}}\mathcode`,="8000 #1\endgroup
}

\usepackage{bbm}
\newcommand{\Item}[1]{\ifx\relax#1\relax  \item \else \item[#1] \fi
	\abovedisplayskip=0pt\abovedisplayshortskip=0pt~\vspace*{-\baselineskip}
}

\usepackage{booktabs}
\usepackage{multirow}

\usepackage{mathtools}
\makeatletter
\newcommand{\raisemath}[1]{\mathpalette{\raisem@th{#1}}}\newcommand{\raisem@th}[3]{\raisebox{#1}{$#2#3$}}
\makeatother

\newcommand{\newcustomtheorem}[2]{\newenvironment{#1}[1]
	{\renewcommand\customgenericname{#2}\renewcommand\theinnercustomgeneric{##1}\innercustomgeneric
	}
	{\endinnercustomgeneric}
}
\newcustomtheorem{customtheorem}{Theorem}
\newcustomtheorem{customlemma}{Lemma}
\newcustomtheorem{customproposition}{Proposition}

\newcommand{\vertiii}[1]{{\vert\kern-0.25ex\vert\kern-0.25ex\vert #1 \vert\kern-0.25ex\vert\kern-0.25ex\vert}}

\newcommand{\beq}{\begin{equation}}
	
	\newcommand{\eeq}{\end{equation}}

\newcommand{\Sc}{\mathcal{S}}

\newcommand{\ubb}{\bar{\vu}}

\newcommand{\bgl}{{~\big |~}}

\definecolor{emmanuel}{RGB}{255,127,0}

\renewcommand{\P}{\operatorname{\mathbb{P}}}

\numberwithin{equation}{section}

\newcommand{\EE}{\mathbb{E}}

\newcommand{\op}{\textup{op}}

\newcommand{\cF}{\mathcal{F}}
\newcommand{\cE}{\mathcal{E}}

\DeclareMathOperator*{\argmin}{argmin}

\newcommand{\RR}{\mathbb{R}}

\newcommand{\cD}{\mathcal{D}}
\newcommand{\cS}{\mathcal{S}}
\newcommand{\cM}{\mathcal{M}}
\newcommand{\cN}{\mathcal{N}}
\newcommand{\cI}{\mathcal{I}}
\newcommand{\cT}{\mathcal{T}}

\newcommand{\cO}{\mathcal{O}}
\newcommand{\cU}{\mathcal{U}}

\newcommand{\vepsilon}{\boldsymbol{\epsilon}}

\begin{document}

\maketitle

\def\thefootnote{$\star$}\footnotetext{equal contribution}\def\thefootnote{\arabic{footnote}}
\setlength{\parindent}{0cm}

\begin{abstract}We consider the problem of learning a realization of a partially observed bilinear dynamical system~(BLDS) from noisy input-output data.
Given a single trajectory of input-output samples, we provide a finite time analysis for learning the system's Markov-like parameters, from which a balanced realization of the bilinear system can be obtained.
Our bilinear system identification algorithm learns the system's Markov-like parameters by regressing the outputs to highly correlated, nonlinear, and heavy-tailed covariates.
Moreover, the stability of BLDS depends on the sequence of inputs used to excite the system.
These properties, unique to partially observed bilinear dynamical systems, pose significant challenges to the analysis of our algorithm for learning the unknown dynamics.
We address these challenges and provide 
high probability error bounds on our identification algorithm
under a uniform stability assumption.
Our analysis provides insights into system theoretic quantities that affect learning accuracy and sample complexity.
Lastly, we perform numerical experiments with synthetic data to reinforce these insights.
\end{abstract}

\section{Introduction}\label{sec:intro} 
Learning the dynamical behavior of nonlinear systems is an important and challenging problem with applications ranging from engineering, physics, biology \cite{brunton2016discovering,strogatz2018nonlinear,brunton2022data}, to language modeling, and sequence predictions \cite{kombrink2011recurrent,bahdanau2014neural}.
Bilinear dynamical systems (BLDS) constitute a simple yet powerful class of nonlinear systems naturally arising in a variety of domains from engineering, biology \cite{bilinearbook}, quantum mechanical processes \cite{pardalos2010optimization} to recommendation systems \cite{koren2021advances}.
Moreover, bilinear systems approximate a much broader class of nonlinear systems via Carleman linearization \cite{kowalski1991nonlinear} or Koopman canonical transform \cite{goswami2017global, bruder2021advantages} of control-affine nonlinear systems \cite{svoronos1980bilinear,Lo1975bilinear}. 
Therefore, learning the dynamics of BLDS from input-output data is an important and useful problem which has attracted significant interest, both in the case of continuous-time \cite{juang2005continuous, sontag2009input} and discrete-time \cite{berk2012identification}. 
However, theoretical guarantees of learning BLDS from a single trajectory of noisy input-output data is lacking, with current guarantees existing only for bilinear systems with complete state observations \cite{sattar2022finite,chatzikiriakos2024end}. 
Our goal in this paper is to provide an algorithm and theoretical guarantees for learning partially observed BLDS from noisy input-output data sampled from a single trajectory. 
We achieve this by learning the system's \emph{Markov-like parameters}, which uniquely identify the end-to-end behavior of the system, and can be used to recover the state-space matrices up to a similarity transform using existing algorithms \cite{ho1966effective,sarkar2019nonparametric,oymak2021revisiting}.

Our goal relates to the problem of learning linear dynamical system~(LDS) from partial state-observations. 
In this setting, 
a line of recent work focuses on finite sample error bounds. \cite{tu2017non,oymak2021revisiting,simchowitz2019learning,sun2022finite,djehiche2022efficient,tsiamis2019finite,sarkar2019finite,bakshi2023new} study methods which use least squares regression to learn the system's Markov parameters or Hankel matrix, which can then be used to recover the state-space matrices (up to a similarity transform) using classic Ho-Kalman Algorithm \cite{ho1966effective}. 
\cite{sun2020finite,sun2022finite, fazel2013hankel} study system identification with Hankel nuclear norm regularization. 
Other works have focused on learning to predict the behavior of partially observed LDS via gradient descent \cite {hardt2018gradient} and spectral filtering \cite{hazan2017learning}.
The linear setting has been extended to (decode-able) nonlinear \cite{mhammedi2020learning} observations and bilinear \cite{sattar2024learning} partial observations.
However, to the best of our knowledge, sample complexity and non-asymptotic analysis for partially observed nonlinear dynamical systems (including BLDS) have not been considered before.

Non-asymptotic learning of (non)linear dynamical systems with complete state observations has also attracted significant attention recently. 
Most of the advancements in this direction are focused on linear systems \cite{faradonbeh2018finite,dean2018regret,simchowitz2018learning, dean2019sample,fattahi2019learning,sarkar2019finite,sarkar2019near,lale2020logarithmic,jedra2020finite,wagenmaker2020active}, where an optimal error rate is achieved by using either mixing-time \cite{yu1994rates} or martingale-based arguments \cite{abbasi2011improved}.
These results have been extended to switched linear dynamical systems \cite{sarkar2019data,sattar2021identification,du2022data,sayedana2024strong}, as well as  certain classes of nonlinear dynamical systems, including state transition models with nonlinear activation functions \cite{oymak2019stochastic,bahmani2019convex,mhammedi2020learning,sattar2020non,jain2021near}, nonlinear features \cite{mania2022active,musavi2024identification}, or from a nonparametric perspective \cite{taylor2021towards,ziemann2022single,kazemian2024random}. 
However, the problem of learning nonlinear dynamical systems from partial observations of a single trajectory is still an open problem.
In this paper, we take a step towards addressing this problem by answering the following question:
\begin{center}
    \emph{Can we learn a partially observed bilinear dynamical system from a single trajectory?}
\end{center}

\noindent The main difficulty arises from the fact that the hidden states evolve according to a bilinear state equation, for which the analysis tools developed for learning partially observed LDS  do not work.
Moreover, the stability of a BLDS explicitly depends on the input sequence, which is in stark contrast to the deterministic notion of stability in the case of LDS. 

\medskip

\noindent {\bf Contributions:} We overcome the aforementioned challenges and provide theoretical guarantees for learning partially observed bilinear dynamical systems. 
We make the following contributions towards bilinear system identification:
\begin{itemize}[leftmargin=*,noitemsep,topsep=0pt]
\item  {\bf Sample complexity and error bounds:}
We provide the first sample complexity analysis and finite-sample error bounds for learning a realization of a partially observed BLDS from a single trajectory of input-output data. 
Unlike LDS, the output of a bilinear system maps to the history of inputs via nonlinear features~(obtained by the Kronecker products of past inputs) and a sequence of Markov-like parameters with exponentially increasing length. 
Our main result (Theorem~\ref{thm:main}) provides $\tilde{\Ocal}(1/\sqrt{T})$ error rate for learning these parameters from a single trajectory of length $T$. 
Our sample complexity bound 
$\tilde{\Omega}\left((p+1)^{L+1}\right)$ grows exponentially with the history length $L$~(where $p$ is the input dimension), which correctly captures the dependence on the number of unknown Markov-like parameters~(growing exponentially with $L$). 
For stable bilinear systems~(defined in \textsection\ref{sec:stability}), this can be mitigated by choosing a smaller history of inputs~(see \textsection~\ref{sec:experiments}).

\item  {\bf Input choice and stability:}
Stability of BLDS is typically input dependent. We work with a novel notion of stability~(Definition~\ref{def:stability}) that generalizes the classic notion of stability for LDS to the BLDS.
We also define a notion of stability radius which governs our choice of inputs.

\item  {\bf Persistence of excitation:}
Of independent interest, we establish the persistence of excitation (Theorem~\ref{thm:persistence}) for a broader class of inputs~(possibly heavy-tailed) satisfying a hyper-contractivity condition.
Our persistence of excitation result holds for nonlinear, correlated, and heavy-tailed covariates~(i.e., the input features).

\item {{\bf Numerical experiments:}}
Lastly, we perform experiments with synthetic data to verify our theoretical findings. Interestingly, our experiments show that exciting the system with inputs sampled uniformly at random from a sphere leads to better estimation of Markov-like parameters as compared to Gaussian inputs, empirically reinforcing our theory on choice of inputs.
\end{itemize}
The rest of the paper is organized as follows: \textsection\ref{sec:problem} sets up the problem and introduces the notion of stability. \textsection\ref{sec:markov_intro} provides our main result on learning Markov-like parameters of partially observed BLDS. 
\textsection\ref{sec:analysis} discusses our proof idea, and provides persistence of excitation result for a broader class of inputs.
Lastly, we perform numerical experiments in \textsection\ref{sec:experiments}, and conclude with a discussion of future directions in \textsection\ref{sec:conclusion}.

\noindent {\bf Notations:} We use boldface lowercase (uppercase) letters to denote vectors (matrices). $\rho(\vX)$, $\|\vX\|_\op$ and $\tf{\vX}$ denote the spectral radius, spectral norm and Frobenius norm of a matrix $\vX$, respectively. 
$\tn{\vv}$ denotes the Euclidean norm of a vector 
$\vv$, and $(\vv)_i$ denotes its $i$-th element. 
For a positive definite matrix $\vM \in \R^{d \times d}$, the Mahalanobis norm of a vector $\vv \in \R^d$ is given by $\norm{\vv}_{\vM} = \sqrt{\vv^\T \vM \vv}$.
For a sequence of $d \times d$ matrices $\Mb_{1}, \dots, \Mb_k$, we use the convention that $\prod_{i=1}^k \Mb_{i} = \Mb_1 \times \Mb_2 \times \cdots 
 \times \Mb_k$. 
We denote by $\Sc^{p-1}$, the centered unit sphere in $\R^p$. 
We denote by $a \vee b$, the maximum of two scalars $a$ and $b$. 
We use $\gtrsim$ and $\lesssim$ for inequalities that hold up to an absolute constant factor. $\tilde{\Ocal}(\cdot)$ and $\tilde{\Omega}(\cdot)$  are used to show the dependence on a specific quantity of interest~(up to constants and logarithmic factors). 
Finally, $\otimes$ denotes the Kronecker product. 

 \section{Preliminaries}\label{sec:problem}
\subsection{Problem Formulation}

Consider a partially observed bilinear dynamical system with the following state-space representation: for all $t \ge 0$, 
\vspace{-6pt}
\begin{equation}
\begin{aligned}\label{eqn:bilinear sys}
	\xb_{t+1} &= \Ab_0 \xb_t + \sum_{k=1}^{p} (\ub_t)_k \Ab_k \xb_t + \Bb\ub_t + \wb_{t},\\
	\vy_t &= \vC \xb_t + \vD \vu_t + \vz_t,
\end{aligned}
\end{equation}
where $\xb_t \in \R^n$, $\ub_t \in \R^p$, and $\vy_t \in \R^m$, $\vw_t \in \R^n$, and $\vz_t \in \R^m$ represent the hidden state, input, output, process noise, and measurement noise, respectively, at time $t$. Without loss of generality, we consider that $x_0=0$. The noise processes $\lbrace \vw_t \rbrace_{t\ge0}$ and $\lbrace \vz_t \rbrace_{t\ge 0}$ are assumed to be sequences of independent, zero-mean, $\sigma^2$-subgaussian random vectors taking values in $\R^n$ and $\R^{m}$, respectively, for some variance proxy parameter $\sigma > 0$. The matrices $\vA_0, \vA_1, \dots, \vA_p \in \RR^{n \times n}$, $\vB \in \RR^{n \times p}$, $\vC \in \RR^{m \times n}$, and $\vD \in \RR^{m \times p}$ represent the parameters that govern the evolution of the dynamical system and are a priori unknown. 

In this work, we wish to identify the unknown parameters of the system from a single trajectory of input-output samples $\lbrace (\vu_t, \vy_t) \rbrace_{t=1}^T$. To that end, we focus on the task of learning the so-called \emph{Markov-like parameters} (detailed in \textsection\ref{sec:markov_intro}) of the system. Once learned, these {Markov-like parameters} can be exploited using the classic Ho-Kalman algorithm to recover the unknown matrices of the system up to some similarity transform as will be described in \textsection\ref{sec:markov_intro}. Next, we  clarify our choice of inputs and discuss the required stability assumption.

\subsection{Input Choice \& Stability of Bilinear Dynamical Systems}\label{sec:stability}

Stability of bilinear dynamical systems is typically input dependent. To see that, we can unroll the state dynamics in \eqref{eqn:bilinear sys} to write: for all $t \ge 0$, 
\begin{align}\label{eqn:bilinear sys state}
        \vx_{t+1} & = \sum_{\ell = 0}^t \left(\prod_{k = 0}^{\ell-1} (\vu_{t-k} \circ \vA)\right) \vB \vu_{t-\ell} + \sum_{\ell = 0}^t \left(\prod_{k = 0}^{\ell-1} (\vu_{t-k} \circ \vA)\right)  \vw_{t-\ell}, 
\end{align}
where we define  $\vu_{t} \circ \vA := \Ab_0 + \sum_{k=1}^{p} (\ub_t)_k \Ab_k$ for the ease of notation. Then, observe that the products of matrices $\prod_{k = 0}^{\ell-1} (\vu_{t-k} \circ \vA)$ may grow exponentially in norm if we consistently choose large inputs. This is precisely why stability in bilinear dynamical systems is more challenging than other classes of systems such as linear dynamical systems or switched systems.

\medskip

Traditionally, notions like Mean Square Stability (MSS) have been considered to reason about the stability behavior of bilinear systems \citep{kubrusly1985mean, pardalos2010optimization, sattar2022finite}. Typically, these notions are asymptotic in nature, require distributional assumptions on the inputs,
permit diverging trajectories with nonzero probability, and may not allow us to obtain tight guarantees. We introduce an alternative notion of stability that naturally generalizes the classical notion of stability in standard LTI systems.

\paragraph{Uniform stability in bilinear dynamical systems:} First, let us recall that the \emph{joint spectral radius} of a set of matrices $\cM \subseteq \RR^{n \times n}$ can be defined as follows:  \begin{align}
    \rho(\cM)  := \lim_{k \to \infty}\sup_{\vM_{1}, \dots, \vM_k \in \cM}  \Vert \vM_{1} \vM_{2} \cdots \vM_{k}\Vert_\op^{1/k}.
\end{align}
For $\rho > 0$, we define the following quantity:
\begin{align}
    \phi(\cM, \rho) := \sup_{k \ge 1, \vM_1, \dots, \vM_k \in \cM} \frac{\Vert \vM_{1} \vM_{2} \cdots \vM_{k}\Vert_\op}{ \rho^k}.
\end{align}
The quantity $\phi(\cM, \rho)$ is defined in similar vein to that by \cite{mania2019certainty} for LDS, and it captures the transient behavior of a system with state transition matrices varying in $\cM$. Note that if $\rho(\cM) < 1$, then for any $\rho > \rho(\cM)$, the quantity $\phi(\cM, \rho)$ is finite. Now, given a set $\cU \subseteq \RR^{p}$, we denote  $\cU \circ \vA := \lbrace  \Ab_0 + \sum_{i=1}^p (\ub)_i \Ab_i: \ub \in \cU \rbrace$ and introduce the following definition of stability.
\begin{definition}[$( \cU, \kappa, \rho)$-uniform-stability]\label{def:stability} Let $\cU \subseteq \RR^p$, $\kappa \ge 1$, and $ 0< \rho < 1$. We say that a partially observed bilinear dynamical system (as defined in \eqref{eqn:bilinear sys}) with state-transition matrices $\vA := \lbrace \vA_0, \dots, \vA_p \rbrace$ is $( \cU, \kappa, \rho)$-uniformly-stable, if the joint spectral radius of the set $\cU \circ \vA$ satisfies: (i) $\rho(\cU \circ \vA) < \rho< 1$; and (ii) $\phi(\cU\circ \vA, \rho)  \le \kappa $.
\end{definition}
Observe that if there exists a nonempty and bounded set $\cU \subseteq \RR^p$ such that $\rho(\cU \circ \vA)< 1$, then for any $\rho(\cU \circ \vA) < \rho < 1$, the system is $(\cU, \kappa, \rho)$-uniformly-stable with $\kappa = \phi(\cU \circ \vA, \rho) \vee 1$. We provide detailed discussion on this claim in Appendix \ref{app:stability}. Furthermore, we note that Definition \ref{def:stability} naturally generalizes that introduced by \cite{monfared2023stabilization}. Indeed, there the authors assume that there exists $\vu^\star$ such that $\rho(\vu^\star \circ \vA) < 1$. This is equivalent to assuming that their system is $(\lbrace \vu^\star\rbrace, \kappa, \rho)$-uniformly-stable for some $\kappa \ge 1$ and $\rho(\vu^\star \circ \vA) <\rho < 1$. As we shall make precise shortly, we need stronger requirements on the stability of the system in comparison with \cite{monfared2023stabilization} because we are concerned with the task of identification. This requirement stems from the need to have persistence of excitation so that estimation is possible. \paragraph{Input choice:} We consider that the inputs $\lbrace \vu_t \rbrace_{t\ge 0}$ are sampled in an i.i.d. manner from some distribution $\cD_{\vu}$ on $\RR^p$. For ease of exposition, we will focus on the case where inputs are sampled uniformly at random from a sphere of radius $\sqrt{p}$, that is, $\vu_t \sim \mathrm{Unif}(\sqrt{p} \cdot\cS^{p-1})$. More generally, as long as the inputs are isotropic and are bounded with high probability, our results will still hold at the expanse of longer proofs. 
\medskip

Putting together this input choice with the stability definition, we are now ready to present the assumption we make on the stability of the bilinear system \eqref{eqn:bilinear sys}.
\begin{assumption}[Stability]\label{ass:stability}
    There exists $\kappa \ge 1$ and $\rho \in (0,1)$ such that the partially observed bilinear dynamical system \eqref{eqn:bilinear sys} is $(\sqrt{p}\cdot\cS^{p-1}, \kappa, \rho)$-uniformly-stable. 
\end{assumption} 
In view of Assumption \ref{ass:stability}, choosing inputs uniformly at random from $\sqrt{p}\cdot \cS^{p-1}$ guarantees stability almost surely. 
More generally, we can choose to sample inputs from any set $\cU$, so long as the system is stable under such set in the sense of Definition \ref{def:stability}. However, the quality of estimation depends on whether inputs sampled from $\cU$ are persistently exciting or not (see \textsection\ref{sec:persistence of excitation}).

\section{Learning the Markov-like Parameters} \label{sec:markov_intro}

The Markov-like parameters are defined in a similar vein (except nonlinear input-output map) to the classical Markov parameters for partially observed LTI systems.  By unrolling the dynamics \eqref{eqn:bilinear sys}, we can represent the output $\vy_t$ in terms of the past $L$ inputs $\vu_{t-L}, \dots, \vu_t$, for any $t \ge L$, as follows:
\begin{equation}\label{eqn:unrolled dynamics}
        \yb_{t}  = \Cb \! \left( \prod_{\ell = 1}^{L} (\ub_{t-\ell}\circ \Ab ) \! \right) \!\xb_{t-L} + \sum_{\ell = 1}^{L} \Cb \!\left( \prod_{i=1}^{\ell-1} (\ub_{t-i} \circ \Ab) \!\right)\! (\Bb \ub_{t-\ell} + \wb_{t-\ell} ) + \Db \ub_t  + \zb_t. 
\end{equation}
We can simplify the form of \eqref{eqn:unrolled dynamics} by adopting a more convenient notation and expanding further some of the products that involve the inputs. First, we introduce $\boldsymbol\epsilon_t := \Cb \! \left( \prod_{\ell = 1}^{L} (\ub_{t-\ell}\circ \Ab ) \! \right) \!\xb_{t-L}$, and define:
\begin{equation}
\begin{aligned}
	 \bar{\vu}_t^\top &:= \begin{bmatrix}
   1  &  \vu_t^\top  
\end{bmatrix}, \qquad
\vutil_t := \begin{bmatrix} \vu_{t} \\
 \vu_{t-1} \\ 
 \ubb_{t-1} \otimes \vu_{t-2} \\ 
 \ubb_{t-1} \otimes \ubb_{t-2} \otimes \vu_{t-3} \\ \vdots \\ 
 \ubb_{t-1} \otimes \ubb_{t-2} \otimes \cdots \otimes \vu_{t-L} \end{bmatrix}, \qquad
	\vwtil_t := \begin{bmatrix} \vz_t \\ \vw_{t-1} \\  \vw_{t-2} \\ 
  \vw_{t-3} \\ \vdots \\ 
  \vw_{t-L} \end{bmatrix}, \\
\text{and} \quad \vF  &:= \begin{bmatrix}
     \Iden_{m} &
    \Cb   &
    \Cb   (\ub_{t-1}\circ \Ab )  &
    \Cb \prod_{\ell = 1}^{2} (\ub_{t-\ell}\circ \Ab )  &
    \hdots &
    \Cb \prod_{\ell = 1}^{L-1} (\ub_{t-\ell}\circ \Ab ) 
 \end{bmatrix}.
 \label{eqn:util_wtil_vec}
\end{aligned}
\end{equation}
Moreover, let us define the matrix $\Gb$ as follows:
\begin{equation}
    \begin{aligned}
    \Gb & := \left[\begin{array}{c c c c c}
         \Db &  \Gb_1 & \Gb_2 & \cdots & \Gb_L
    \end{array} \right]  \in \R^{m \times d_{\util}}, \quad \text{with} \quad  d_{\util} = (p+1)^{L} + p-1,
\end{aligned}
\end{equation}
where $\Gb_1 {=} \Cb\Bb $, $\Gb_{\ell} {=} \lbrace \Cb \Ab_{i_1} \times \cdots \times \Ab_{i_{\ell-1}} \Bb  \rbrace_{i_1, \dots, i_{\ell -1}  \in \lbrace 0, \dots, p\rbrace} \in \RR^{m \times p(p+1)^{\ell - 1}}$, for $\ell \in \lbrace 2, \dots, L \rbrace$. The parameters $\lbrace \Cb \Bb, \lbrace\Cb \Ab_{i_1} \vB\rbrace_{i_1 \in \lbrace 0, \dots, p\rbrace}, \dots, \lbrace\Cb \Ab_{i_1} \times \cdots \times \Ab_{i_{L-1}} \Bb \rbrace_{i_1, \dots, i_{L -1}  \in \lbrace 0, \dots, p\rbrace} \rbrace $ are what we refer to as the \emph{Markov-like parameters} of the system. We are finally ready to rewrite~\eqref{eqn:unrolled dynamics} in terms of these parameters as follows: for all $t\ge L$, we have  
\begin{align}\label{eq:input-output form}
    \yb_t =  \Gb \vutil_t + \Fb \vwtil_{t} + \boldsymbol\epsilon_t.  
\end{align}
With the dynamics written in the form of \eqref{eq:input-output form}, it is natural to use the least squares estimation method to learn the Markov-like parameters from the observations $\lbrace \vy_t, \vu_t \rbrace_{t=1}^{T}$ . More specifically, the (minimum norm) Least Squares Estimator (LSE) $\vGhat$ of $\Gb$ admits a closed form and can be defined as:
\begin{align} \label{eqn:ERM_Ghat}
	\vGhat := \left( \sum_{t=L}^T   \vy_t \vutil_{t}^\top \right) \left(\sum_{t=L}^{T} \vutil_{t} \vutil_{t}^\top \right)^{\dagger}   & \in \argmin_{\vG \in \R^{m \times d_{\util}}} \sum_{t=L}^{T} \tn{\vy_t - \vG \vutil_{t}}^2. 
\end{align}
Moreover, when the matrix $\sum_{t=L}^{T} \vutil_{t} \vutil_{t}^\top \succ 0$, then estimation error can be expressed as:
 \begin{align}\label{eq:estimation error}
     \vGhat  - \vG = \left( \sum_{t=L}^T (\vF\vwtil_t  +  \boldsymbol\epsilon_t) \vutil_t ^\top \right) \left(\sum_{t=L}^{T} \vutil_{t} \vutil_{t}^\top \right)^{-1} .
 \end{align}
We now present Theorem \ref{thm:main}, our main result on the recovery of the \emph{Markov-like parameters}:
\begin{theorem}[Learning Markov-like parameters]\label{thm:main}
    Let $\delta \in (0,1)$, $T \ge 0$. Suppose Assumption \ref{ass:stability} holds, and the inputs are sampled uniformly at random from a sphere of radius $\sqrt{p}$, that is, $\{\vu_t\}_{t \geq 0} \distas \mathrm{Unif}(\sqrt{p} \cdot\cS^{p-1})$. Then the event: \begin{align}\label{eq:error bound}
         \Vert \vGhat  - \vG \Vert_\op  \le \frac{C (\kappa^2\rho^{L} + \kappa)}{1- \rho}\sqrt{\frac{L p^2(p+1)^{(L+1)} \left( \log\left(\frac{e L}{\delta}\right)  +m + n L + (p+1)^{L+1}\right) }{T-L}} 
    \end{align}
    holds with probability at least $1-\delta$, provided that 
\begin{align}
    T - L &\gtrsim  L(L+1) \left(\frac{3p}{p + 2} \right)^{L+1}\left (\log\left(\frac{(L+1)}{\delta}\right) + (p+1)^{L+1} \log\left(\frac{(p+1)^{L+1}}{\delta}\right) \right), \label{eqn:trajectory_size_main_thm}
\end{align}
with positive constant $C = \poly(\sigma, \Vert \vB \Vert_\op, \Vert \vC\Vert_\op)$. 
\end{theorem}

We discuss the analysis of the estimation error leading to Theorem \ref{thm:main} in \textsection\ref{sec:analysis}. There, we also highlight the key challenges and steps in establishing this result. From the bound in \eqref{eq:error bound}, we see the recovery error $\Vert \vGhat - \vG \Vert_\op $ scales as, ignoring all other dependencies,  $\tilde{\cO} (\sqrt{(p+1)^{2(L+1)}/(T-L)})$. This contrasts with partially observed linear systems where typically we only have a polynomial dependence in $L$, and also reflects the difficulty in learning bilinear systems from partial observations. To recover the unknown matrices $\vC, \vA_0, \dots, \vA_p, \vB$, we require $L$ large enough, typically larger than $2n$ (see Remark \ref{rem:param recov}).

\begin{remark}[The BLDS parameter recovery.]\label{rem:param recov}
We remark  that, for every  $k \in \lbrace 0, \dots, p\rbrace$, we can directly extract from $\vGhat$, estimates of the matrices 
$\lbrace \vC\vB, \vC \vA_k \vB, \dots,\vC \vA_k^{L-1} \vB \rbrace$. To see that, observe that letting $\cI_{k,\ell} \subset \lbrace 1, \dots, d_{\vutil} \rbrace$ be the $p$ indices corresponding to the $p$-dimensional sub-vector of $\vutil_t$, $(\prod_{i=1}^{\ell-1} (\vu_{t-i})_k) \vu_{t-\ell}$, then $\vG_{:, \cI_{k,\ell}} = \vC \vA_{k}^{\ell-1} \vB$. In other words, we can take $\vGhat_{:, \cI_{\ell, k}}$ to be an estimate $\vC \vA_{k}^{\ell-1} \vB$. We then construct a Hankel matrix from $\vGhat_{:, \cI_{\ell, k}}$. Under the condition that our estimation error is sufficiently small, each pair $(\vA_k, \vB)$ is controllable, each pair $(\vA_k,\vC)$ is observable, and $L \geq 2n$, we can use classic Ho-Kalman algorithm~\citep{ho1966effective} to estimate $\vA_0, \vA_1, \dots, \vA_p, \vB, \vC$ up to a similarity transform, with robustness guarantees~\citep{oymak2021revisiting}. Lastly, note that the estimate of $\vD$ is obtained as the first $p$ columns of $\vGhat$.
\end{remark}

 \section{Sample Complexity Analysis} \label{sec:analysis}

To prove Theorem \ref{thm:main}, we start our analysis by decomposing the estimation error as follows:
\begin{align}\label{eq:error decomposition}
    \left\Vert \hat{\vG} - \vG  \right\Vert_\op  \le \underbrace{\left\Vert \left(\sum_{t=L}^{T} \vutil_t \vutil_t^\T \right)^{\dagger} \right\Vert_\op}_{\textrm{Excitation}}  \Bigg(\underbrace{\left\Vert \sum_{t=L}^T \vutil_t (\vF\tilde{\vw}_t)^\T \right\Vert_\op}_{\textrm{Mutiplier Process}} +   \underbrace{\left\Vert \sum_{t=L}^T \vutil_t \vepsilon_t^\T \right\Vert_\op}_{\textrm{Truncation Bias} } \Bigg),
\end{align}
where we use the submultiplicativity of $\Vert \cdot\Vert_\op$ and the triangular inequality. Next, we will analyze each of three terms appearing in the decomposition above separately and obtain corresponding bounds in high probability. Once these bounds have been established, the proof concludes immediately (see details in Appendix~\ref{app:proof main}). In what follows, we focus on presenting the results regarding the analysis of the three terms. 
We note that the challenge in analyzing this terms lies in the presence of non-trivial dependencies and nonlinearities. As such,  recent analysis tools from the non-asymptotic system identification literature \citep{ziemann2023tutorial} do not apply, and this is precisely what we manage to tackle.

\subsection{Persistence of Excitation}\label{sec:persistence of excitation}
We show persistence of excitation which is necessary to ensure that the LSE is a consistent estimator. More precisely, we will establish that smallest singular value of the design matrix $\vUtil$ whose rows correspond to $\{\vutil_t^\T\}_{t=L}^T$ is bounded from below  by $\tilde{\Omega}(\sqrt{T-L+1})$. One of the major sources of difficulty in establishing this persistence of excitation result challenging is the nonlinear dependence of $\vutil_{t}$  on $\vu_{t-L}, \dots, \vu_{t}$ for all $t \ge L$. We need to understand how distributional properties of the input impact the lower spectrum of $\vUtil^\T \vUtil$. To that end, we start by introducing the property of hypercontractivity.

\begin{definition}[Hypercontractivity]\label{def:hypercont}
A $p$-dimensional random vector $\vu$ is  $(4,2, \gamma)$-hypercontractive, if  $\EE[(\vu^\top \vx)^4] \le \gamma \EE[(\vu^\top \vx)^2]^2$, for all $\vx \in \RR^p$. 
\end{definition}
The $(4,2,\gamma)$-hypercontractivity property is satisfied by many classical distributions. Notably, a $p$-dimensional standard gaussian random vectors satisfies it with $\gamma = 3$, while $p$-dimensional random vectors sampled uniformly from $\sqrt{p} \cdot \cS^{p-1}$ satisfies $\gamma {=} 3/(1+2/p)$. We refer the reader to Appendix \ref{app:persistence excitation} for a proof to these claims.
\begin{assumption}[Distributional properties of the input] \label{ass:input distribution}
    $\lbrace \vu_t \rbrace_{t\ge 0}$ is a sequence of independent zero-mean, isotropic\footnote{A $p$-dimensional random vector $\vu$ is isotropic if for all $x \in \RR^p$, $\EE[(\vu^\top x)^2] = \Vert x \Vert_{\ell_2}^2$.}, and $(4, 2, \gamma)$-hypercontractive for some $\gamma > 1$, $p$-dimensional random vectors with zero third moment marginals\footnote{A $p$-dimensional random vector $\vu$ has zero-third-moment-marginals if for all $\vx \in \RR^p$, $\EE[(\vu_t^\top \vx)^3] = 0$.}.
\end{assumption}
Again, it can be verified that Assumption \ref{ass:input distribution} is satisfied by inputs sampled from $\cN(0, \Iden_p)$ or $\mathrm{Unif}(\sqrt{p} \cdot \cS^{p-1})$. More importantly, Assumption \ref{ass:input distribution} covers a wide range of input distributions that may even be heavy-tailed, as it only requires conditions on the first four moments of the distribution. This contrast with classical assumptions that require the input distribution to have sub-Gaussian tails, and is also consistent with the intuition that bounding the smallest singular value of random matrix requires weaker moment conditions  \citep{koltchinskii2015bounding}. We are now ready to present our main result on the  persistence of excitation:

\begin{theorem}[Persistence of Excitation]\label{thm:persistence}
Suppose the sequence of inputs $\lbrace \vu_{t}\rbrace_{t \ge 0}$ are selected as per Assumption~\ref{ass:input distribution}, then for all $\delta \in (0,1)$, the event: 
\begin{align}
    \lambda_{\min}\left(\vUtil^\T \vUtil\right) =  \lambda_{\min}\left( \sum_{t=L}^{T} \vutil_t \vutil_t^\T \right) \geq (T-L+1)/4.
\end{align}
holds with probability at least $1- \delta$, provided that 
\begin{align}
    T &\gtrsim L +  L(L+1) (3 \vee \gamma)^{L+1}\left (\log\left(\frac{(L+1)}{\delta}\right) + (p+1)^{L+1} \log\left(\frac{(p+1)^{L+1}}{\delta}\right) \right). \label{eqn:trajectory_size_main}
\end{align}
\end{theorem}
The proof of Theorem \ref{thm:persistence} is deferred to Appendix \ref{app:persistence excitation}. Interestingly, despite the presence of nonlinearities and dependencies in the covariates of $\vUtil$, persistence of excitation is still guaranteed. Part of the reason is because the distributional properties presented in Assumption \ref{ass:input distribution} ensure that isotropy of the vectors $\vutil_t$ is still preserved, and their third and fourth moments are well bounded. The dependence in $(p+1)^{L+1}$ in \eqref{eqn:trajectory_size_main} is unavoidable because of the dimension of the vectors $\vutil_t$.

\subsection{Analysis of the Multiplier Process}

The analysis of the the multiplier process term $\Vert \sum_{t = L}^{T} \vutil_t(\vF\vwtil_t)^\T\Vert_\op$ is somewhat simpler than that of truncation bias term. The reason is because the sequences $\lbrace \vutil_{t}\rbrace_{t\ge L}$, and $\lbrace \vwtil_t \rbrace_{t\ge L}$ are independent with zero-mean vectors. However, each of these two sequences contains dependent vectors. Below, we present a high probability bound on the multiplier process term showing that we can still control this term despite the presence of these dependencies: 
\begin{proposition}\label{prop:multiplier process}
    Let $\delta \in (0,1)$ and $T \ge L$. The event: 
    \begin{align*}
        \left\Vert \sum_{t = L}^{T} \vutil_t(\vF\vwtil_t)^\T\right\Vert_\op  \le C_2 \frac{\kappa}{1-\rho}  \sqrt{ L (T-L)(p+1)^{L+1} \left(\log\left(\frac{eL}{\delta}\right)  + m + nL + (p+1)^{L+1} \right) }
    \end{align*}
    holds with probability $1-\delta$,  
    with a positive constant $C_2 = \poly(\sigma, \Vert \vC\Vert_\op)$.
\end{proposition}
The key idea behind the proof of Proposition \ref{prop:multiplier process} is observing that each of the subsequences, $\ell \in \lbrace 1, \dots, L\rbrace$,  $\lbrace \vwtil_{kL + \ell} \rbrace_{k \ge 0}$, has independent vectors. Thus, we can use a blocking trick to rewrite the multiplier process as a sum of $L$ martingales which can bound using classical concentration tools. This argument is made precise in the proof and is deferred to Appendix \ref{app:multiplier process}.

\subsection{Analysis of the Truncation Bias}

The analysis of the truncation bias term $\Vert \sum_{t=L}^T \vutil_{t} \vepsilon_t^\top  \Vert_\op$ is challenging for multiple reasons. Indeed, first, the sequences $\lbrace \vutil_{t}\rbrace_{t \ge L}$ and $\lbrace \vepsilon_{t} \rbrace_{t \ge L}$ are non-trivially dependent, and second, $\vepsilon_t$ is only zero-mean conditioned on future inputs $\vu_{t - L}, \dots, \vu_T$ and is still dependent on $\vu_{0}, \dots, \vu_{t-L-1}$. It is worth noting that this type of dependence does not arise when learning partially observed linear dynamical systems. Indeed, the analysis of the truncation bias term in linear systems involves $\vepsilon_t$ that are only dependent on the covariates $\vx_{t-L}$ and are independent of $\vu_{t-L},  \dots,  \vu_{t}$ (e.g., see \cite{oymak2019stochastic}). In addition to that, when learning linear systems $L$ can be made large enough to trivially bound the truncation bias, whereas in bilinear systems our bounds pay exponential dependence on $L$ so choosing this parameter more carefully is more important. 
Nonetheless, we establish a high probability bound on the truncation bias as presented below: 

\begin{proposition}\label{prop:truncation bias}
    Let $\delta \in (0,1)$ and $T \ge L$. The event:
        \begin{align*}
        \left\Vert \sum_{t=L}^T \vutil_t \vepsilon_t^\top \right\Vert_\op \le  \frac{C_1 \kappa^2 \rho^L }{1-\rho}\sqrt{(T-L)p^2(p+1)^{L+1}\left( \log\left(\frac{e}{\delta}\right) + n + (p+1)^{L+1} \right)}      
    \end{align*}
    holds with probability at least $1-\delta$, 
    with a positive constant $C_1 = \poly(\sigma, \Vert \vB \Vert_\op, \Vert \vC \Vert_\op)$.
\end{proposition}
The proof of Proposition \ref{prop:truncation bias} relies on the critical observation that the truncation bias term can be rewritten as martingale, namely as follows: for all $\theta \in \cS^{d_{\vutil} - 1}, \lambda \in \cS^{m - 1}$, 
\begin{align}
    \theta^\top\left(\sum_{t=L}^T \vutil_t \vepsilon_t^\top\right)\lambda =  \sum_{s=0}^{T-L-1} (\vB \vu_s + \vw_s)^\top \vf_s(\theta, \lambda, ,\vu_{s+1}, \dots, \vu_{T}), 
\end{align}
where the functions $\lbrace \vf_{s} \rbrace_{s \ge 0}$ are nonlinear in their arguments. Moreover, thanks to our choice of inputs and the stability Assumption \ref{ass:stability}, the terms $\vf_s(\theta, \lambda, ,\vu_{s+1}, \dots, \vu_{T})$ are bounded and do not scale with $T$. Thus, we use classical concentration tools to deduce our final bounds. The details of the proofs including the precise definition of $\lbrace \vf_{s} \rbrace_{s \ge 0}$ are delayed until Appendix \ref{app:truncation bias}.

\section{Numerical Experiments} \label{sec:experiments}

For our experiments, we consider a partially observed bilinear dynamical system~\eqref{eqn:bilinear sys} with $n {=} 5$ hidden states, input dimension $p {=} 2$, and output dimension $m{=}2$. The dynamics matrices $\vA_0, \vA_1, \vA_2$ are constructed with i.i.d. $\Ncal(0,1)$ entries, and are scaled to have spectral radius $\rho(\vA_0) {=} \rho_0$ and $\rho(\vA_1) {=} \rho(\vA_2) {=} \rho_k$, where $\rho_0,\rho_k$ are hyper-parameters in our experiments. Similarly, $\vB,\vC$ and $\vD$ are generated with i.i.d. $\Ncal(0, 1/n)$ and $\Ncal(0,1/m)$ entries, respectively. The noise processes $\{\vw_t\}_{t=0}^T$, and $\{\vz_t\}_{t=0}^T$ are chosen according to i.i.d. Gaussian distribution with zero mean and variances $\sigma^2 = 0.0001$. Lastly, the control inputs are either sampled uniformly at random from the sphere $\sqrt{p}\cdot\cS^{p-1}$ or sampled i.i.d. from a Gaussian distribution $\Ncal(0, \vI_p)$.

In Figure~\ref{figure1}, we plot the estimation error $\norm{\vG - \vGhat}_{\op}^2$ over different values of $\rho_0, \rho_k, L$ and $T$. Each experiment is repeated $10$ times and we plot the mean and one standard deviation. Figure~\ref{fig1a} and Figure~\ref{fig1c} 
correspond to estimation with Gaussian inputs $\{\vu_t\}_{t\geq 0} \distas \Ncal(0, \Iden_p)$, whereas, Figure~\ref{fig1b} and Figure~\ref{fig1d} correspond to estimation with uniformly distributed inputs $ \{\vu_t\}_{t\geq 0} \distas \mathrm{Unif}(\sqrt{p}\cdot \cS^{d-1})$.

Figure~\ref{figure1} shows that the estimation error increases with $L$, because the number of Markov-like parameters increases exponentially in $L$. More interestingly, Figure~\ref{figure1} suggests that choosing the inputs $\{\vu_t\}_{t\geq 0} \distas \mathrm{Unif}(\sqrt{p}\cdot \cS^{d-1})$ leads to more accurate estimation than $\{\vu_t\}_{t\geq 0} \distas \Ncal(0, \Iden_p)$ at any given $T$. This seems to highlight the role of hypercontractivity parameter $\gamma$ in our estimation problem. 
Recall that $\gamma = 3$ for Gaussian inputs, whereas, $\gamma = \frac{3}{1+2/p}$ for the uniformly distributed input.

We observe double descent curves~\citep{nakkiran2020optimal} for $L{=}7$. This is because our regression problem is unregularized and has $(p+1)^L + p - 1 = 2188$ unknown parameters, and the number of input-features~($\vutil_t$) is $T-7$. Hence, for $L=7$, we see the peak at $T=2195$, and the error decays smoothly after this point. Note that the peak occurs at $T-L = (p+1)^L + p - 1$~(where the number of unknown parameters become equal to the number of input-features). For $L = \{5,6\}$, we do not see the double descent because  we start at $T=1000$ which is greater than $(p+1)^L + p - 1$.

\begin{figure}[t!]
    \begin{centering}
        \begin{subfigure}[t]{1.5in}
            \includegraphics[width=\linewidth]{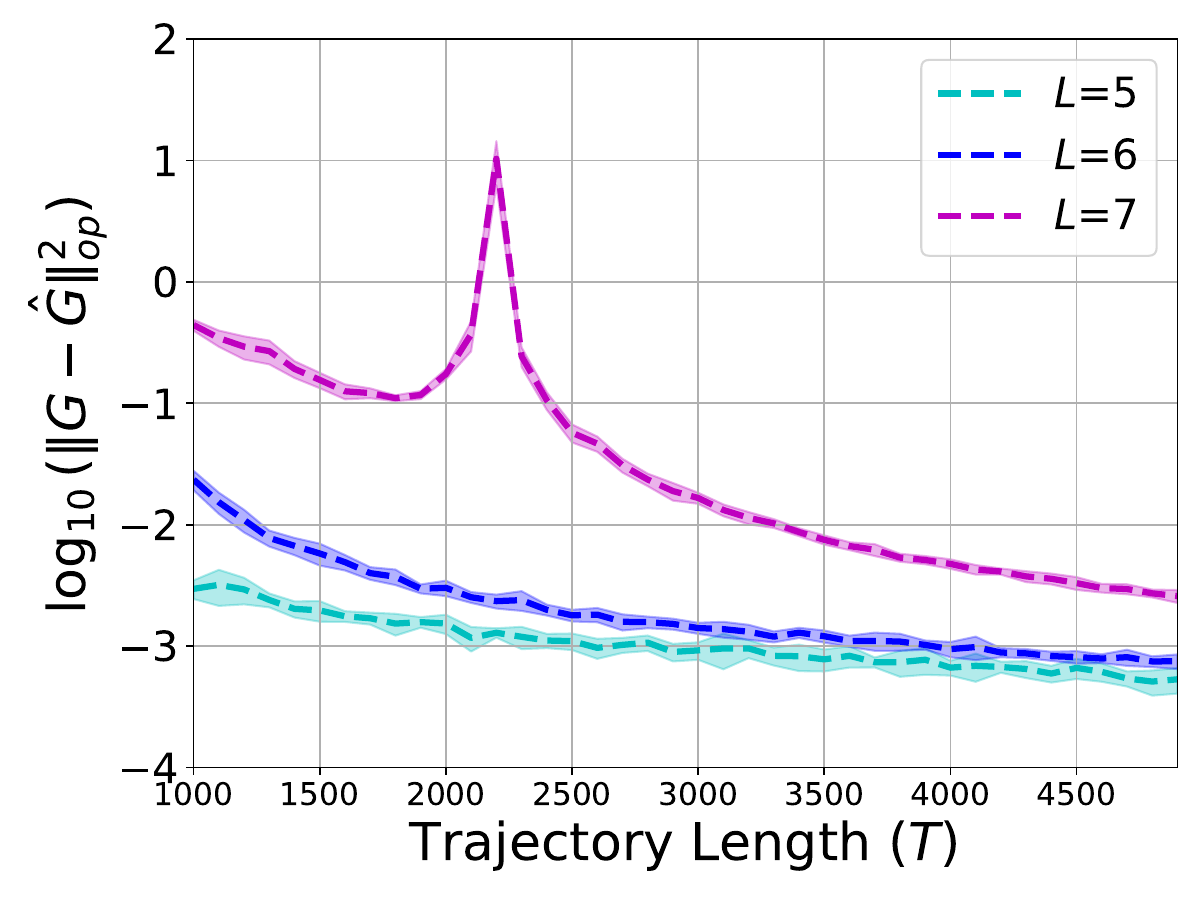}
\caption{$\{\vu_t\}_{t\geq 0} {\distas} \Ncal(0, \Iden_p)$,  $\rho_0 {=} 0.4,\rho_k {=} 0.2$}\label{fig1a}
        \end{subfigure}
    \end{centering}
    ~
    \begin{centering}
        \begin{subfigure}[t]{1.5in}
            \includegraphics[width=\linewidth]{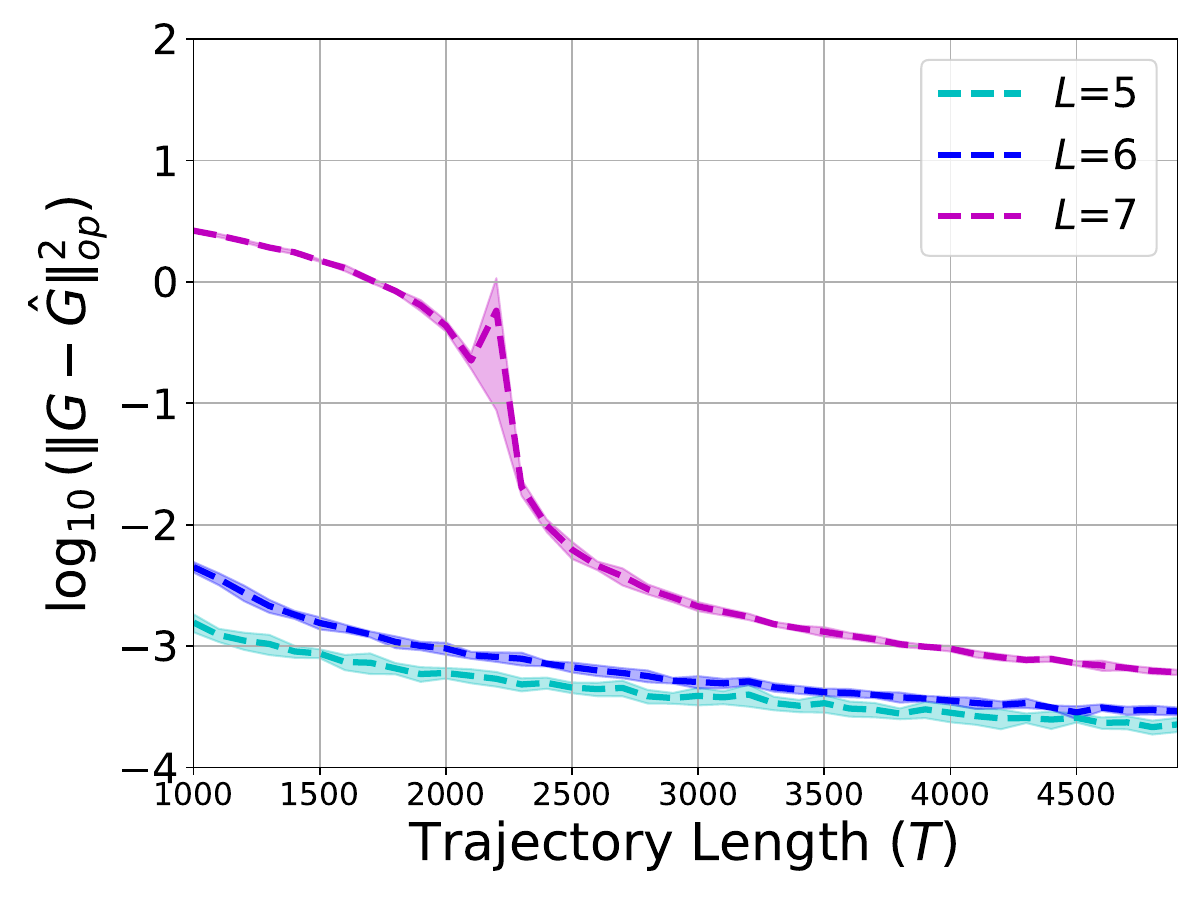}
\caption{$\{\vu_t\}_{t\geq 0} {\distas} \mathrm{Unif}(\sqrt{p}\cdot \cS^{d-1})$,  $\rho_0 {=} 0.4,\rho_k {=} 0.2$}\label{fig1b}
        \end{subfigure}
     \end{centering}
     ~
    \begin{centering}
        \begin{subfigure}[t]{1.5in}
            \includegraphics[width=\linewidth]{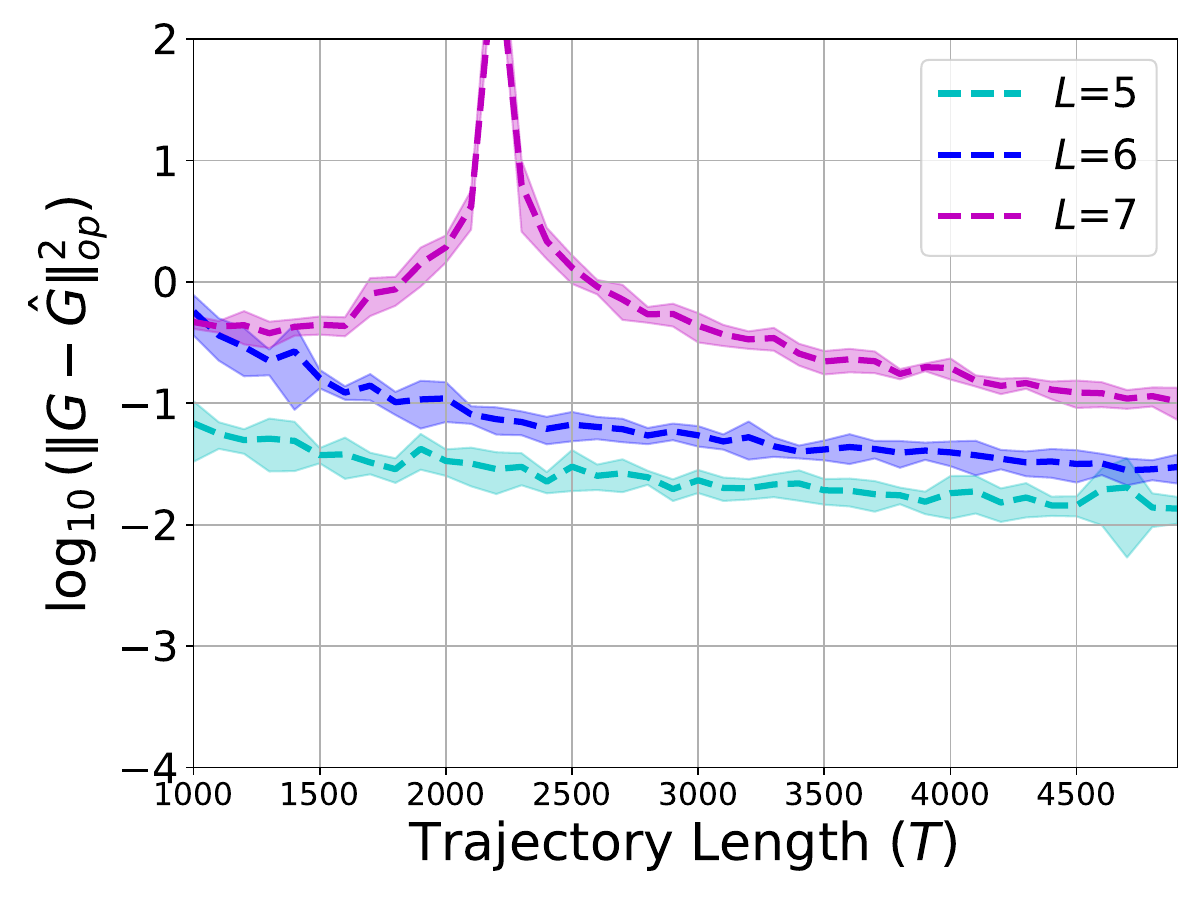}
\caption{$\{\vu_t\}_{t\geq 0} {\distas} \Ncal(0, \Iden_p)$,  $\rho_0 {=} 0.6,\rho_k {=} 0.4$}\label{fig1c}
        \end{subfigure}
    \end{centering}
    ~
    \begin{centering}
        \begin{subfigure}[t]{1.5in}
            \includegraphics[width=\linewidth]{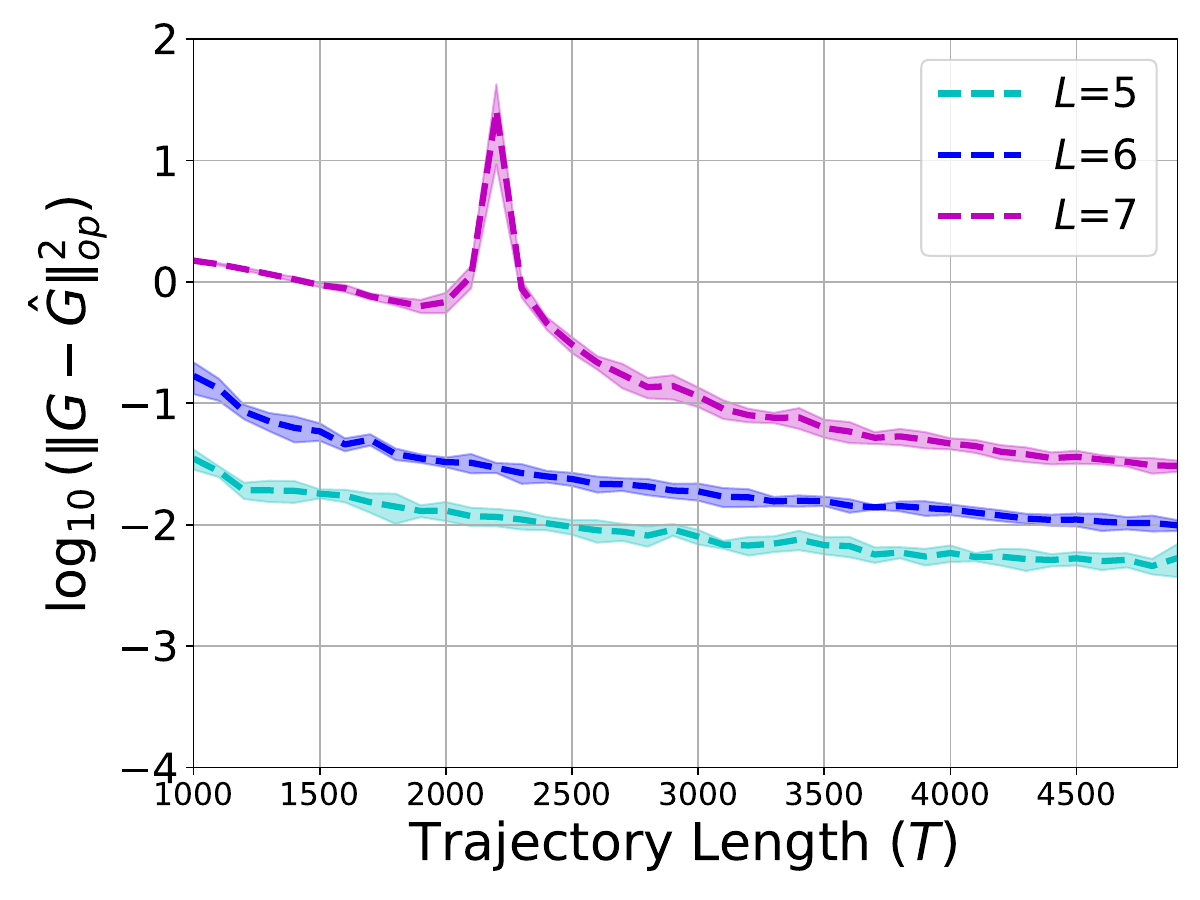}
\caption{$\{\vu_t\}_{t\geq 0} {\distas} \mathrm{Unif}(\sqrt{p}\cdot \cS^{d-1})$,  $\rho_0 {=} 0.6,\rho_k {=} 0.4$}\label{fig1d}
        \end{subfigure}
    \end{centering}
    \caption{ We plot the estimation error $\norm{\vG - \vGhat}_{\op}^2$ over different values of $T$, $L$, $\rho(\vA_0)$, $\rho(\vA_1)$, $\rho(\vA_2) $ while fixing $n{=}5$, $p{=}2$ and $m{=}2$. Figure~\ref{fig1a} and Figure~\ref{fig1c} correspond to estimation with Gaussian inputs $\{\vu_t\}_{t\geq 0} \distas \Ncal(0, \Iden_p)$, whereas, Figure~\ref{fig1b} and Figure~\ref{fig1d} correspond to estimation with uniformly distributed inputs $ \{\vu_t\}_{t\geq 0} \distas \mathrm{Unif}(\sqrt{p}\cdot \cS^{d-1})$. Our plots show that the later input choice gives better estimation of $\vG$ as compared to the first one. Moreover, the estimation error increases as $L$ and $\rho$ increases, whereas, it decreases as $T$ increases. }
    \label{figure1}
\end{figure}

\section{Conclusion and Future Direction}\label{sec:conclusion}
We provide the first non-asymptotic learning bounds for partially observed BLDS. 
Given finite input-output data sampled from a single trajectory of BLDS, we learn its Markov-like parameters, provide an upper bound on the estimation error with $\tilde{\Ocal}(1/\sqrt{T})$ dependence, and provide a bound on the number of samples required, scaling as $\tilde{\Omega}\left( (p+1)^{L+1}\right)$. These parameters uniquely characterize the input-output map of a BLDS via nonlinear input features, hence, can be used to recover the state-space matrices.
Our results hold under a novel notion of stability that generalizes the classic notion of stability for LDS to the BLDS.

\medskip 

There are several interesting future directions. First, can the exponential dependence on $L$ be avoided? We believe this can be done by carefully designing the inputs such that the number of Markov-like parameters do not grow exponentially in $L$. Second, can our results be extended to account for marginally-stable BLDS? This requires stabilization of a partially observed BLDS with unknown state-space matrices which itself is an interesting future direction. Other possible directions include exploring the benefit of regularization (e.g., Hankel nuclear norm regularization) for BLDS identification, exploring gradient-based methods for learning partially observed BLDS, and adaptive control of BLDS.

\section*{Acknowledgements}
S.D. was partly supported by NSF CCF 2312774, NSF OAC-2311521, NSF IIS-2442137, a PCCW Affinito-Stewart Award, a Gift to the LinkedIn-Cornell Bowers CIS Strategic Partnership, and an AI2050 Early Career Fellowship program at Schmidt Sciences. M.F. was supported in part by awards NSF TRIPODS II 2023166, NSF CCF 2007036, NSF CCF 2212261, NSF CCF 2312775, and by an Amazon-UW Hub gift award. Y.J. was partially supported by the Knut and Alice Wallenberg Foundation Postdoctoral Scholarship Program at MIT-KAW 2022.0366.

\newpage 
\bibliographystyle{alpha}
\bibliography{Bibfiles}

\newpage 

\appendix
\section{Stability of Bilinear Dynamical Systems} \label{app:stability}

In this appendix, we present some results that concerns the stability of bilinear system in the sense of Definition \ref{def:stability}. In Lemma \ref{lem:stability notion}, we present a condition under which a bilinear system satisfies uniform stability. In Lemma \ref{lem:upper bound stability}, we present a bound on $\Vert \vF\Vert_\op$ which follows under our stability assumptions.

\begin{lemma}\label{lem:stability notion}
    Let $\cU$ be a bounded and non-empty subset of $\RR^{p}$ such that $\rho(\cU \circ \vA) < 1$, then for all $ \rho \in (\rho(\cU \circ \vA), 1)$, there exists $\kappa \ge 1$ such that the  system is $(\cU , \kappa, \rho)$-uniformly stable.
\end{lemma}

\begin{proof}
    Let $\rho \in (\rho(\cU \circ \vA), 1)$. We recall that the joint spectral radius of $\cU \circ \vA$ is defined as follows: 
    \begin{align}
        \rho(\cU \circ \vA ) = \lim_{k \to \infty} \sup_{\vM_1, \dots, \vM_k \in \cU \circ \vA} \Vert \vM_1 \vM_2 \cdots \vM_k \Vert_\op^{1/k}. 
    \end{align}
    Since by assumption the limit exists, we have by definition that 
    \begin{align*}
         \forall \epsilon > 0, \ \ \exists k_0 \ge 1, \ \  \forall k \ge k_0, \quad  \left\vert \sup_{\vM_1, \dots, \vM_k \in \cU \circ \vA }\Vert \vM_1 \vM_2 \cdots \vM_k \Vert_\op^{1/k} - \rho(\cU \circ \vA) \right\vert < \epsilon \rho(\cU \circ \vA).  
ˇ    \end{align*}
    Thus choosing $\epsilon > 0$ such that $\rho \ge  (1+\epsilon) \rho(\cU \circ \vA)$, we can find $k_0$ such that \begin{align*}
        \forall k \ge k_0, \qquad \sup_{\vM_1, \dots, \vM_k \in \cU \circ \vA }\Vert \vM_1 \vM_2 \cdots \vM_k \Vert_\op <  \rho^k.
    \end{align*}
    Now, we can further define 
    \begin{align}
        \kappa = \max\left\lbrace 1, \sup_{1 \le k \le k_0}\sup_{ \vM_1, \dots, \vM_k \in \cU \circ \vA }\frac{\Vert \vM_1 \vM_2 \cdots \vM_k \Vert_\op^{1/k}}{\rho}  \right\rbrace.
    \end{align}
    We note that $\kappa$ is well defined because $\cU \circ \vA$ is a bounded set of matrices since $\cU$ is bounded. Indeed, for all $\vM \in \cU \circ \vA$, we have $\Vert \vM\Vert_\op \le \max\lbrace 1, \sup_{\vu \in \cU }\Vert \vu \Vert_{\ell_\infty}\rbrace (\Vert \vA_0 \Vert_\op  + \dots + \Vert \vA_p \Vert_\op )$. This concludes the proof.
\end{proof}

\begin{lemma}\label{lem:upper bound stability}
Suppose Assumption \ref{ass:stability} holds and 
let $(\vu_t)_{t \ge 0}$ be a sequence of inputs taking values in $\sqrt{p} \cdot\cS^{p-1}$, and . We have 
\begin{align}\label{eq:stability first bound}
    \forall t \ge 0, \qquad \left\Vert \prod_{\ell = 0}^t (\vu_\ell \circ \vA )\right\Vert_\op \le \kappa \rho^{t+1}.
\end{align}
Consequently, we have \begin{align}\label{eq:stability second bound}
    \left\Vert \vF  \right\Vert_\op \le   1 + \frac{\kappa \Vert \vC \Vert_\op  }{1- \rho}
\end{align}
\end{lemma}
\begin{proof}
The first inequality \eqref{eq:stability first bound} holds immediately thanks to stability. The second inequality \eqref{eq:stability second bound} is an immediate consequence of \eqref{eq:stability first bound}. Indeed, we have 
\begin{align*}
        \Vert \vF \Vert_\op \le 1 + \Vert \vC\Vert_\op \sum_{\ell = 1}^{L-1} \left\Vert \prod_{i = 1}^{\ell -1} (\vu_{t- i} \circ A) \right\Vert_\op \le 1 + \Vert \vC\Vert_\op \sum_{\ell = 1}^{L-1} \kappa \rho^{\ell - 1} \le 1 + \frac{\kappa \Vert \vC \Vert_\op}{1- \rho}  
\end{align*} 

\end{proof}

\section{Proofs for Persistence of Excitation}\label{app:persistence excitation}

In this appendix, we will present our results on persistence of excitation under the following assumption unless specified otherwise.

\begin{definition}
    A $p$-dimensional random vector $\vu$ is \emph{isotropic} if for all $\vx \in \RR^p$, $\EE[(\vu^\top \vx)^2] = \Vert \vx \Vert_{\ell_2}^2$.
    It has \emph{zero third moment marginals} if for all $\vx \in \RR^p$, $\EE[(\vu_t^\top \vx)^3] = 0$.
\end{definition}

\begin{assumptionp}{\ref{ass:input distribution}}[Distributional properties of the input]
    $\lbrace \vu_t \rbrace_{t\ge 0}$ is a sequence of independent zero-mean, isotropic, and $(4, 2, \gamma)$-hypercontractive for some $\gamma > 1$, $p$-dimensional random vectors with zero third moment marginals.
\end{assumptionp}
There are many choice of input distribution that satisfy Assumption \ref{ass:input distribution}. Notably, 
if the inputs are sampled independently according to $\cN(0, I_p)$ or $\mathrm{Unif}(\sqrt{p} \cdot \cS^{p-1})$, then Assumption \ref{ass:input distribution} holds with $\gamma = 3$ or $\gamma = \frac{3}{1+2/p}$ respectively. Before we present the proof of our main result on persistence of excitation, we present some intermediate results in the next two subsections.

\subsection{Hypercontractivity to Bounded Moments (Input)}
\begin{lemma}\label{lemma:hyper_to_moment}
     Let $\vQ, \vR \in \RR^{p \times d}$, and $\gamma>0$. Let $\vu$ be an isotropic and $(4,2,\gamma)$-hypercontractive random vector taking values in $\RR^{p}$. We have that 
    $$
    \EE[(\vu^\top \vQ \vQ^\top \vu)(\vu^\top \vR \vR^\top \vu)] \le \gamma  \Vert \vQ \Vert^2_F \Vert \vR \Vert_F^2 
    $$
\end{lemma}

\begin{proof}
    Note that if $\vu$ is isotropic and $(4,2, \gamma)$-hypercontractive, then for any $\vQ, \vR \in \RR^{p \times d}$, we have 
\begin{align*}
    \EE[(\vu^\top \vQ \vQ^\top \vu)(\vu^\top \vR \vR^\top \vu)] & = \EE\left[\left( \sum_{i=1}^d (\vu^\top \vq_i)^2 \right) \left( \sum_{i=1}^d (\vu^\top \vr_i)^2 \right)  \right], \\
    & = \EE\left[ \sum_{1 \le i,j \le d} (\vu^\top \vq_i)^2(\vu^\top \vr_j)^2  \right], \\
    & \leqsym{a} \sum_{1 \le i,j \le d}\sqrt{\EE\left[ (\vu^\top \vq_i)^4   \right]} \sqrt{\EE\left[  (\vu^\top \vr_j)^4  \right]}, \\
    & \leqsym{b} \sum_{1 \le i,j \le d} \sqrt{\gamma} \EE\left[ (\vu^\top \vq_i)^2   \right] \sqrt{\gamma}\EE\left[  (\vu^\top \vr_j)^2  \right], \\
    & \le \gamma \sum_{1 \le i,j \le d}  \Vert \vq_i \Vert^2_{\ell_2}  \Vert \vr_j \Vert_{\ell_2}^2,  \\
    & \le \gamma  \Vert \vQ \Vert^2_F \Vert \vR \Vert_F^2,  
\end{align*}
where we obtain (a) by using Cauchy-Schwarz inequality, and (b) from the $(4,2, \gamma)$-hypercontractivity assumption. This completes the proof.  
\end{proof}

\subsection{Hypercontractivity to Bounded Moments (Covariates)}
\begin{lemma}\label{lemma:hyper_to_moment_cov}
     Suppose the sequence of inputs $\{\vu_t\}_{t \geq 0}$ satisfy Assumption~\ref{ass:input distribution}. Let $\vutil_t$ be the covariates (input features) as defined in \eqref{eqn:util_wtil_vec}. We have that
     \begin{align*}
         \E[(\vutil_t^\T \vv)^2] = 1 \quad \text{and} \quad \E[(\vutil_t^\T \vv)^4] \leq  L(3 \vee \gamma)^{L+1}, 
     \end{align*}
     for all $t = L, L+1, \dots, T$ and all $\vv \in \Sc^{d_{\util}-1}$
\end{lemma}
\begin{proof}
    To begin, recall the definition of $\vutil_t$ from \eqref{eqn:util_wtil_vec}. For notational convenience, we define another random features of inputs $\vutiltil_t$ as follows,
\begin{align}
	\vutil_t := \begin{bmatrix} \vu_{t} \\ \vu_{t-1} \\ \ubb_{t-1} \otimes \vu_{t-2} \\ \ubb_{t-1} \otimes \ubb_{t-2} \otimes \vu_{t-3} \\ \vdots \\ \ubb_{t-1} \otimes \ubb_{t-2} \otimes \cdots \otimes \vu_{t-L} \end{bmatrix},  \quad 
    \vutiltil_t := \begin{bmatrix} \ubb_{t-1} \\ \ubb_{t-1} \otimes \ubb_{t-2} \\ \ubb_{t-1} \otimes \ubb_{t-2} \otimes \ubb_{t-3} \\ \vdots \\ \ubb_{t-1} \otimes \ubb_{t-2} \otimes \cdots \otimes \ubb_{t-L} \end{bmatrix}, \label{eqn:util_vec}
\end{align}
where $\ubb_t = [1~\vu_t^\T]^\T$. Note that, $\vutil_t \in \R^{d_{\util}}$, and $\vutiltil_t \in \R^{d_{\utiltil}}$,  
\begin{align}
	d_{\util} &= p+\sum_{i = 0}^{L-1} p(p+1)^i = p + p\frac{(p+1)^L-1}{(p+1)-1} = (p+1)^L + p -1, \nn \\
 \text{and} \quad d_{\utiltil} &= \sum_{i = 0}^{L-1} (p+1)(p+1)^i = (p+1)\frac{(p+1)^L-1}{(p+1)-1} = \frac{(p+1)^{L+1}-(p+1)}{p}. \nn
\end{align}
We first show that, under Assumption~\ref{ass:input distribution}, $\vutil_t$ is isotropic as follows.

\medskip
\noindent {\bf Isotropic Covariates:} To begin, note that $\E[\vutil_t] = 0$ due to Assumption~\ref{ass:input distribution}. 
Next, we will show that $\vSigma[\vutil_t] = \E[\vutil_t\vutil_t^\T] =  \Iden_{d_{\util}}$ as follows: Let $[\vutil_t]_0 = \vu_t$, and let $[\vutil_t]_i$ be the $i$-th partition of $\vutil_t$ with $p(p+1)^{i-1}$ entries for $i = 1,\dots, L$. 
Then, we have that $\E\left[[\vutil_t]_0[\vutil_t]_0^\T\right] = \E\left[[\vutil_t]_1[\vutil_t]_1^\T\right] = \Iden_p$. Similarly, for $i = 2, \dots,L$, we have
\begin{align}
	\E\left[[\vutil_t]_i[\vutil_t]_i^\T\right] &= \E\left[(\ubb_{t-1} \otimes \cdots \otimes \vu_{t-i})(\ubb_{t-1} \otimes \cdots \otimes \vu_{t-i})^\T\right], \nn \\
    &= \E\left[\ubb_{t-1} \ubb_{t-1}^\T \otimes \cdots \otimes \vu_{t-i}\vu_{t-i}^\T\right], \nn \\
	&= \Iden_{p+1} \otimes \cdots \otimes \Iden_{p} = \Iden_{p(p+1)^{i-1}}.  \label{eqn:mean_diag}
\end{align}
Hence, we show that $\E\left[[\vutil_t]_i[\vutil_t]_i^\T\right] = \Iden_{p(p+1)^{(i-1) \vee 0}}$ for all $i = 0, \dots,L$. Next, we will show that $\E\left[[\vutil_t]_i[\vutil_t]_j^\T\right] = 0$, for all $i,j = 0, \dots,L$ when $i \neq j$. Because of symmetry, it suffices to consider $i < j$. In this case, we have
\begin{align}
	\E\left[[\vutil_t]_i[\vutil_t]_j^\T\right] &= \E\left[(\ubb_{t-1} \otimes \cdots \otimes \vu_{t-i})(\ubb_{t-1} \otimes \cdots \otimes \vu_{t-j})^\T\right], \nn \\
	&= \E\left[(\ubb_{t-1} \otimes \cdots \otimes \vu_{t-i})(\ubb_{t-1} \otimes \cdots \otimes \ubb_{t-i})^\T \right]\otimes \E\left[(\ubb_{t-i-1} \otimes \cdots \otimes \vu_{t-j} )^\T\right], \nn \\ 
&= 0, \label{eqn:mean_off_diag}
\end{align}
for all $2\leq i<j$. Similarly, we also have $\E\left[[\vutil_t]_1[\vutil_t]_j^\T\right] = \E\left[\vu_{t-1}(\ubb_{t-1} \otimes \cdots \otimes \vu_{t-j})^\T\right] = 0$, for all $j = 2, \dots,L$. Lastly, $[\vutil_t]_0$ is independent of $[\vutil_t]_j$, for all $j = 1, \dots,L$, due to Assumption~\ref{ass:input distribution}, hence, $\E\left[[\vutil_t]_0[\vutil_t]_j^\T\right] = 0$.
Combining \eqref{eqn:mean_diag} and \eqref{eqn:mean_off_diag}, we get
\begin{align}
	\E [\vutil_t \vutil_t^\T] =  \Iden_{d_{\util}} \quad \implies \quad  \E[(\vutil_t^\T \vv)^2] = 1,
\end{align}
for all $\vv \in \Sc^{d_{\util}-1}$. This gives the first statement of Lemma~\ref{lemma:hyper_to_moment_cov}. Next, we show that $\vutil_t$ also have bounded fourth moment marginals as follows:

\medskip
\noindent {\bf Fourth Moment Marginals:} Let $\vv \in \Scal^{d_{\vutil} - 1}$, and consider the random variable $(\vv^\T \vutil_t)^2$. 
In the following we will upper bound $\E [(\vv^\T \vutil_t)^4]$.  To begin,  let $\vv_i$ denote the $i$-th partition of $\vv$ with $p(p+1)^{(i-1) \vee 0}$ entries, for $i = 0,\dots, L$. Then, $\vv$ can be represented as,
\begin{align}
	\vv := \begin{bmatrix}
		\vv_0^\T & \vv_1^\T  & \vv_2^\T & \hdots & \vv_{L}^\T
	\end{bmatrix}^\T.
\end{align}
Hence, we have
\begin{align}
	\E \left[ \left( \vv^\T \vutil_t \right)^4\right] &= \E \left[ \left(\sum_{i=0}^L \vv_i^\T [\vutil_t]_i  \right)^4\right], \nn \\
    &= \E \left[ \left(\vv_0^\T \vu_{t} + \vv_1^\T \vu_{t-1} + \sum_{i=2}^L \vv_i^\T (\ubb_{t-1}  \otimes \cdots \otimes \vu_{t-i})  \right)^4\right]. \label{eqn:4th_moment_expansion_full}
\end{align}
In the following, we will upper each term appearing on the right-hand-side~(RHS) of \eqref{eqn:4th_moment_expansion_full} to get an upper bound on the fourth moment marginals of $\vutil_t$.

\medskip
$\bullet$ {\bf Linear Terms:}
For $i \neq j \neq k \neq \ell$, consider the following expectation: We utilize $\E[\vu_{t-\max\{i,j,k,\ell\}}] =0$, to show that
\begin{align}
	&\E\left[(\vv_i^\T [\vutil_t]_i)(\vv_j^\T [\vutil_t]_j)(\vv_k^\T [\vutil_t]_k)(\vv_\ell^\T [\vutil_t]_\ell)\right] = 0. \end{align}

\medskip
$\bullet$ {\bf Quadratic Terms I:} For $i \neq j \neq k$, consider the following expectation: When $i < \max\{j,k\}$, we can utilize $\E[ \vu_{t-\max\{j,k\}} ]=0$, to show that
\begin{align}
	\E\left[(\vv_i^\T [\vutil_t]_i)^2 (\vv_j^\T [\vutil_t]_j)(\vv_k^\T [\vutil_t]_k)\right] &= 0. \end{align}
However, when $i > \max\{j,k\}$, the above expectation is not zero. We can upper bound these terms using Cauchy-Schwarz inequality as follows,
\begin{align}
    \E\left[(\vv_i^\T [\vutil_t]_i)^2 (\vv_j^\T [\vutil_t]_j)(\vv_k^\T [\vutil_t]_k)\right] 
    & \leq \sqrt{\E\left[(\vv_i^\T [\vutil_t]_i)^4 \right] \E\left[(\vv_j^\T [\vutil_t]_j)^2(\vv_k^\T [\vutil_t]_k)^2\right]}.
\label{eqn:quadratic_one_split}
\end{align}
In the remaining of the proof, we will upper bound the two terms on the RHS of \eqref{eqn:quadratic_one_split} individually, along-with the remaining terms on the RHS of \eqref{eqn:4th_moment_expansion_full}.

\medskip
$\bullet$ {\bf Quadratic Terms II:} For $i \neq j$, consider the expectation $\E\left[(\vv_i^\T [\vutil_t]_i)^2 (\vv_j^\T [\vutil_t]_j)^2\right]$. In order to upper bound this term, we first upper bound an intermediate term as follows: 
Let $[\vutiltil_t]_k$ denote the $k$-th partition of $\vutiltil_t$ defined in~\eqref{eqn:util_vec}, for $k=1, \dots,L$. Let $\vq_k, \vq'_k$ be two arbitrary vectors in $\vR^{(p+1)^k}$. We will use the mathematical induction to prove the following intermediate result under Assumption~\ref{ass:input distribution},
\begin{align}
    \E\left[(\vq_k^\T [\vutiltil_t]_k)^2(\vq_k'^\T [\vutiltil_t]_k)^2\right]   \leq  (3 \vee \gamma)^k \tn{\vq_k}^2\tn{\vq_k'}^2, \quad \text{for all} \quad k = 1, \dots,L. \label{eq:forth_cross_tiltil}
\end{align}
We begin the proof, by showing that, $k=1$ obeys the induction as follows,
\begin{align}
	&\E\left[(\vq_1^\T [\vutiltil_t]_1)^2(\vq_1'^\T [\vutiltil_t]_1)^2\right] \nn \\
    &\qquad = \E \left[\left(\vq_1^\T \ubb_{t-1} \right)^2\left(\vq_1'^\T \ubb_{t-1} \right)^2 \right], \nn \\
    &\qquad\eqsym{i} \E \left[\left(q_{11}+ \vq_{12}^\T \vu_{t-1} \right)^2\left(q_{11}' + \vq_{12}'^\T \vu_{t-1} \right)^2 \right], \nn \\
    &\qquad= \E \left[\left(q_{11}^2+ (\vq_{12}^\T \vu_{t-1})^2 + 2q_{11}\vq_{12}^\T \vu_{t-1} \right)\left(q_{11}'^2 + (\vq_{12}'^\T \vu_{t-1})^2 + 2q_{11}'\vq_{12}'^\T \vu_{t-1} \right) \right], \nn \\
    &\qquad\leqsym{ii} q_{11}^2q_{11}'^2 + q_{11}^2 \tn{\vq_{12}'}^2 + q_{11}'^2 \tn{\vq_{12}}^2 + \gamma \tn{\vq_{12}}^2\tn{\vq_{12}'}^2 + 4 q_{11} q_{11}' \tn{\vq_{12}}\tn{\vq_{12}'}, \nn \\
    &\qquad\leqsym{ii} q_{11}^2q_{11}'^2 + 3q_{11}^2 \tn{\vq_{12}'}^2 + 3q_{11}'^2 \tn{\vq_{12}}^2 + \gamma \tn{\vq_{12}}^2\tn{\vq_{12}'}^2, \nn \\
    &\qquad\leq (3 \vee \gamma)\left(q_{11}^2 + \tn{\vq_{12}}^2 \right)\left(q_{11}'^2 + \tn{\vq_{12}}'^2 \right) = (3 \vee \gamma)\tn{\vq_1}^2\tn{\vq_1'}^2, \label{eqn:induction_step1_cross}
\end{align}
where we obtain (i) from setting $\vq_1 = [q_{11}~~\vq_{12}^\T]^\T$, $\vq_1' = [q_{11}'~~\vq_{12}'^\T]^\T$, (ii) from applying Lemma~\ref{lemma:hyper_to_moment} along-with Cauchy–Schwarz inequality, and (iii) is obtained by using the identity $2ab \leq a^2+b^2$ for $a, b \in \R$.
Suppose we have $\E\left[\left(\vq_{k-1}^\T [\vutiltil_t]_{k-1}\right)^2  \left(\vq_{k-1}'^\T [\vutiltil_t]_{k-1}\right)^2 \right] \leq  (3 \vee \gamma)^{k-1} \tn{\vq_{k-1}}^2\tn{\vq_{k-1}'}^2 $. Then, we apply the induction as follows,
\begin{align}
    \E\left[\left(\vq_{k}^\T [\vutiltil_t]_{k}\right)^2  \left(\vq_{k}'^\T [\vutiltil_t]_{k}\right)^2 \right] &= \E \left[\left(\vq_k^\T ([\vutiltil_t]_{k-1} \otimes \ubb_{t-k}) \right)^2 \left(\vq_k'^\T ([\vutiltil_t]_{k-1} \otimes \ubb_{t-k}) \right)^2 \right], \nn \\
    &=  \E \left[\left( \ubb_{t-k}^\T \vQ_k [\vutiltil_t]_{k-1} \right)^2 \left( \ubb_{t-k}^\T \vQ_k' [\vutiltil_t]_{k-1} \right)^2 \right] , \nn \\
	&=  \E \left[\E \left[\left( \ubb_{t-k}^\T \vQ_k [\vutiltil_t]_{k-1} \right)^2 \left( \ubb_{t-k}^\T \vQ_k' [\vutiltil_t]_{k-1} \right)^2  \bgl \vu_{t-k} \right] \right], \nn \\
	& \leqsym{i} (3 \vee \gamma)^{k-1}\E \left[ \tn{\vQ_k^\T \ubb_{t-k}}^2 \tn{\vQ_k'^\T \ubb_{t-k}}^2 \right], \nn \\
&\leqsym{ii} (3 \vee \gamma)^k \tf{\vQ_k}^2 \tf{\vQ_k'}^2, \nn \\
    &= (3 \vee \gamma)^k \tn{\vq_k}^2 \tn{\vq_k'}^2,
\end{align}
where $\vQ_k = \mat(\vq_k) \in \R^{(p+1) \times (p+1)^{k-1}}$, $\vQ_k' = \mat(\vq_k') \in \R^{(p+1) \times (p+1)^{k-1}}$, we obtain (i) from the induction hypothesis, and (ii) follows from a similar line of reasoning as used to derive an upper bound in \eqref{eqn:induction_step1_cross}. This completes the proof of our intermediate result in \eqref{eq:forth_cross_tiltil}. Next, we use \eqref{eq:forth_cross_tiltil} to derive an upper bound on $\E\left[(\vv_i^\T [\vutil_t]_i)^2 (\vv_j^\T [\vutil_t]_j)^2\right]$ as follows: Due to symmetry, it is sufficient to consider $j>i$. We begin by deriving the upper bound for $j> i \geq 2$ as follows,

\begin{align}
	&\E\left[(\vv_i^\T [\vutil_t]_i)^2 (\vv_j^\T [\vutil_t]_j)^2\right] \nn \\
    &\qquad = \E \left[\left(\vv_i^\T (\ubb_{t-1} \otimes \cdots \otimes \vu_{t-i})\right)^2 \left(\vv_j^\T (\ubb_{t-1} \otimes \cdots \otimes \vu_{t-j})\right)^2 \right], \nn \\
	&\qquad= \E \left[\E \left[\left(\vv_i^\T (\ubb_{t-1} \otimes \cdots \otimes \vu_{t-i})\right)^2 \left(\vv_j^\T (\ubb_{t-1} \otimes \cdots \otimes \vu_{t-j})\right)^2 \bgl \vu_{t-i}, \vu_{t-i-1}, \dots, \vu_{t-j} \right]\right], \nn \\
	&\qquad\eqsym{i}  \E \big[ \E \big[\left(\vu_{t-i}^\T\vV_i (\ubb_{t-1} \otimes \cdots \otimes \ubb_{t-i+1})\right)^2 \left((\ubb_{t-i}^\T \otimes \cdots \otimes \vu_{t-j}^\T)\vV_j (\ubb_{t-1} \otimes \cdots \otimes \ubb_{t-i+1})\right)^2 \nn \\
    & \qquad\qquad \bgl \vu_{t-i}, \vu_{t-i-1}, \dots, \vu_{t-j} \big]\big], \nn \\
    &\qquad=  \E \left[ \E \left[\left(\vu_{t-i}^\T\vV_i [\vutiltil_t]_{i-1}\right)^2 \left((\ubb_{t-i}^\T \otimes \cdots \otimes \vu_{t-j}^\T)\vV_j [\vutiltil_t]_{i-1}\right)^2 \bgl \vu_{t-i}, \vu_{t-i-1}, \dots, \vu_{t-j} \right]\right], \nn \\
	& \qquad \leqsym{ii} (3 \vee \gamma)^{i-1} \E \left[ \tn{\vV_i^\T \vu_{t-i}}^2  \tn{\vV_j^\T (\ubb_{t-i} \otimes \cdots \otimes \vu_{t-j})}^2  \right], \label{eqn:cross_moment_after_intermediate}
\end{align}
where we obtain (i) from tower rule, and setting $\vV_i = \mat(\vv_i) \in \R^{p \times (p+1)^{i-1}}$, $\vV_j = \mat(\vv_j) \in \R^{p(p+1)^{j-i} \times (p+1)^{i-1}}$, and (ii) is obtained by using the intermediate result from \eqref{eq:forth_cross_tiltil}. To proceed, observe that
\begin{align}
    \ubb_{t-i} \otimes \cdots \otimes \vu_{t-j} &= \begin{bmatrix}
        \ubb_{t-i-1} \otimes \cdots \otimes \vu_{t-j} \\
        \vu_{t-i} \otimes \ubb_{t-i-1} \otimes \cdots \otimes \vu_{t-j}
    \end{bmatrix}, \quad \vV_j  = \begin{bmatrix}
        \vV_{j1} \\
        \vV_{j2}
    \end{bmatrix},
\end{align}
where $\vV_{j1} \in \R^{p(p+1)^{j-i-1} \times (p+1)^{i-1}}$, and $\vV_{j2} \in \R^{p^2(p+1)^{j-i-1} \times (p+1)^{i-1}}$. Combining this with \eqref{eqn:cross_moment_after_intermediate}, we have
\begin{align}
	\E\left[(\vv_i^\T [\vutil_t]_i)^2 (\vv_j^\T [\vutil_t]_j)^2\right] &\leq (3 \vee \gamma)^{i-1} \E \big[ \tn{\vV_i^\T \vu_{t-i}}^2  \|\vV_{j1}^\T (\ubb_{t-i-1} \otimes \cdots \otimes \vu_{t-j}) \nn \\ 
    &+ \vV_{j2}^\T (\vu_{t-i} \otimes \ubb_{t-i-1} \otimes \cdots \otimes \vu_{t-j})\|_{\ell_2}^2  \big], \nn \\
    &\leq (3 \vee \gamma)^{i-1} \E \big[ \tn{\vV_i^\T \vu_{t-i}}^2  \big(\tn{\vV_{j1}^\T (\ubb_{t-i-1} \otimes \cdots \otimes \vu_{t-j})}^2 \nn \\
    &+ \tn{\vV_{j2}^\T (\vu_{t-i} \otimes \ubb_{t-i-1} \otimes \cdots \otimes \vu_{t-j})}^2 \nn \\ 
    &+ 2(\ubb_{t-i-1} \otimes \cdots \otimes \vu_{t-j})^\T \vV_{j1}\vV_{j2}^\T (\vu_{t-i} \otimes \ubb_{t-i-1} \otimes \cdots \otimes \vu_{t-j}) \big) \big], \label{eqn:quadraticII_split}
\end{align}
In the following, we will upper bound each term in \eqref{eqn:quadraticII_split} separately to get an upper bound on $\E\left[(\vv_i^\T [\vutil_t]_i)^2 (\vv_j^\T [\vutil_t]_j)^2\right]$. To begin, we have
\begin{align}
    &\E \big[ \tn{\vV_i^\T \vu_{t-i}}^2 \tn{\vV_{j1}^\T (\ubb_{t-i-1} \otimes \cdots \otimes \vu_{t-j})}^2\big] \nn \\ 
    &\qquad= \E \big[ \vu_{t-i}^\T\vV_i\vV_i^\T \vu_{t-i}\big] \E \big[(\ubb_{t-i-1} \otimes \cdots \otimes \vu_{t-j})^\T\vV_{j1}\vV_{j1}^\T (\ubb_{t-i-1} \otimes \cdots \otimes \vu_{t-j}) \big], \nn \\
    &\qquad= \E \big[ \tr \big(\vV_i\vV_i^\T \vu_{t-i}\vu_{t-i}^\T\big)\big] \E \big[ \tr \big(\vV_{j1}\vV_{j1}^\T (\ubb_{t-i-1} \otimes \cdots \otimes \vu_{t-j})(\ubb_{t-i-1} \otimes \cdots \otimes \vu_{t-j})^\T\big) \big], \nn \\
    &\qquad \leq \tf{\vV_i}^2\tf{\vV_{j1}}^2. \label{eqn:quadraticII_split_1solved}
\end{align}
Similarly, the second term in \eqref{eqn:quadraticII_split} can be upper bounded as follows,

\begin{align}
    &\E \big[ \tn{\vV_i^\T \vu_{t-i}}^2 \tn{\vV_{j2}^\T (\vu_{t-i} \otimes \ubb_{t-i-1} \otimes \cdots \otimes \vu_{t-j})}^2\big] \nn \\
    &\qquad \eqsym{a} \sum_{k = 1}^{(p+1)^{i-1}} \E \big[ \tn{\vV_i^\T \vu_{t-i}}^2 \big(\vv_{j2k}^\T (\vu_{t-i} \otimes \ubb_{t-i-1} \otimes \cdots \otimes \vu_{t-j})\big)^2\big], \nn \\
    &\qquad \eqsym{b}  \sum_{k = 1}^{(p+1)^{i-1}} \E \big[ \tn{\vV_i^\T \vu_{t-i}}^2 \big((\ubb_{t-i-1} \otimes \cdots \otimes \vu_{t-j})^\T\vV_{j2k}\vu_{t-i}\big)^2\big], \nn \\
    &\qquad =  \sum_{k = 1}^{(p+1)^{i-1}} \E \big[ \E \big[ \vu_{t-i}^\T\vV_i\vV_i^\T \vu_{t-i} \big((\ubb_{t-i-1} \otimes \cdots \otimes \vu_{t-j})^\T\vV_{j2k}\vu_{t-i}\big)^2 \bgl \vu_{t-i-1},\dots,\vu_{t-j} \big] \big], \nn \\
    &\qquad \leqsym{c} \sum_{k = 1}^{(p+1)^{i-1}}  \gamma \tf{\vV_i}^2\E \big[ \tr\big( (\ubb_{t-i-1} \otimes \cdots \otimes \vu_{t-j})^\T\vV_{j2k} \vV_{j2k}^\T(\ubb_{t-i-1} \otimes \cdots \otimes \vu_{t-j}) \big) \big], \nn \\
    &\qquad =  \gamma \tf{\vV_i}^2 \sum_{k = 1}^{(p+1)^{i-1}} \tf{\vV_{j2k}}^2, \nn \\
    &\qquad = \gamma \tf{\vV_i}^2 \tf{\vV_{j2}}^2, \label{eqn:quadraticII_split_2solved}
\end{align}
where we obtain (a) by defining $\vv_{j2k}^\T$ to be the $k$-th row of $\vV_{j2}^\T$, (b) from setting $\vV_{j2k} = \mat(\vv_{j2k}) \in \R^{p(p+1)^{j-i-1} \times p}$, and (c) follows from the application of Lemma~\ref{lemma:hyper_to_moment}. Lastly, note that the third term in \eqref{eqn:quadraticII_split} is zero as follows,
\begin{align}
    &\E \left[ \tn{\vV_i^\T \vu_{t-i}}^2 (\vu_{t-i} \otimes \ubb_{t-i-1} \otimes \cdots \otimes \vu_{t-j})^\T\vV_{j2} \vV_{j1}^\T(\ubb_{t-i-1} \otimes \cdots \otimes \vu_{t-j}) \right] \nn \\
    &\qquad= \E \left[ \tn{\vV_i^\T \vu_{t-i}}^2 (\ubb_{t-i-1} \otimes \cdots \otimes \vu_{t-j})^\T \mat\left(\vV_{j2} \vV_{j1}^\T(\ubb_{t-i-1} \otimes \cdots \otimes \vu_{t-j})\right) \vu_{t-i} \right], \nn \\
    &\qquad= \E \big[\E \big[ \tn{\vV_i^\T \vu_{t-i}}^2 (\ubb_{t-i-1} \otimes \cdots \otimes \vu_{t-j})^\T \mat\left(\vV_{j2} \vV_{j1}^\T(\ubb_{t-i-1} \otimes \cdots \otimes \vu_{t-j})\right) \vu_{t-i}\nn \\
    & \qquad \qquad \bgl \vu_{t-i-1}, \dots, \vu_{t-j} \big]\big], \nn \\
    &\qquad= 0. \label{eqn:quadraticII_split_3solved}
\end{align}
Finally, combining \eqref{eqn:quadraticII_split_1solved}, \eqref{eqn:quadraticII_split_2solved}, and \eqref{eqn:quadraticII_split_3solved} into \eqref{eqn:quadraticII_split} we get the following upper bound,
\begin{align}
    \E\left[(\vv_i^\T [\vutil_t]_i)^2 (\vv_j^\T [\vutil_t]_j)^2\right] &\leq (3 \vee \gamma)^{i-1} \left(\tf{\vV_i}^2\tf{\vV_{j1}}^2 +  \gamma \tf{\vV_i}^2 \tf{\vV_{j2}}^2\right), \nn \\
    & \leq (3 \vee \gamma)^i \tn{\vv_i}^2\tn{\vv_j}^2, \quad \text{for all} \quad j> i \geq 2.
\end{align}
For $j>i=0$, it is easy to see that $\E\left[(\vv_0^\T [\vutil_t]_0)^2 (\vv_j^\T [\vutil_t]_j)^2\right] = \E\left[(\vv_0^\T \vu_t)^2\right] \E\left[(\vv_j^\T [\vutil_t]_j)^2\right] = \tn{\vv_0}^2\tn{\vv_j}^2$. Finally, for $j>i=1$, we get the following upper bounds,
\begin{align}
    \E\left[(\vv_1^\T [\vutil_t]_1)^2 (\vv_j^\T [\vutil_t]_j)^2\right] &= \E\left[(\vv_1^\T \vu_{t-1})^2(\vv_j^\T (\ubb_{t-1} \otimes \cdots \otimes \vu_{t-j}))^2\right], \nn \\
    &= \E \left[ \E\left[(\vv_1^\T \vu_{t-1})^2((\ubb_{t-2} \otimes \cdots \otimes \vu_{t-j})^\T \vV_j \ubb_{t-1})^2 \bgl \vu_{t-1}\right] \right], \nn \\
    &= \E \left[(\vv_1^\T \vu_{t-1})^2 \tn{\vV_j \ubb_{t-1}}^2\right], \nn \\
    &\eqsym{i} \E \left[(\vv_1^\T \vu_{t-1})^2 \tn{ \vv_{j1}+ \vV_{j2} \vu_{t-1}}^2\right], \nn \\
    &\leqsym{ii} \tn{\vv_1}^2\tn{\vv_{j1}}^2 + \gamma \tn{\vv_1}^2 \tf{\vV_{j2}}^2 \leq \gamma \tn{\vv_1}^2\tn{\vv_j}^2,
\end{align}
where we obtain (i) from setting $\vV_{j} {=} \mat(\vv_j) {:=} [\vv_{j1}~~\vV_{j2}] \in \R^{p(p+1)^{j-2} \times (p+1)}$, and (ii) follows from Lemma~\ref{lemma:hyper_to_moment}. Hence, for all $i \neq j$, we have
\begin{align}
	\E\left[(\vv_i^\T [\vutil_t]_i)^2 (\vv_j^\T [\vutil_t]_j)^2\right] \leq (3 \vee \gamma)^{\min\{i,j\}} \tn{\vv_i}^2 \tn{\vv_j}^2, \quad \text{for all} \quad i \neq j. 
\end{align}

$\bullet$ {\bf Cubic Terms:} For $i \neq j$, consider the following expectation: When $i < j$, we can utilize $\E [ \vu_{t-j}]  =0$, and when $i > j$, we can use $\E[(\vb^\T \vu_{t-i})^3] = 0$, to show that
\begin{align}
	\E\left[(\vv_i^\T [\vutil_t]_i)^3 (\vv_j^\T [\vutil_t]_j)\right] = 0. \end{align}

$\bullet$ {\bf Quartic Terms:} Consider the expectation $\E\left[(\vv_i^\T [\vutil_t]_i)^4\right]$. In order to upper bound this term, we first upper bound an intermediate term as follows: Recall the definition of $\vutiltil_t$ from \eqref{eqn:util_vec}, and let $[\vutiltil_t]_k$ denote its $k$-th partition with $(p+1)^k$ entries, for $k = 1,\dots, L$. For any $\vq_k \in \vR^{(p+1)^k}$, we will use the mathematical induction to prove that 
\begin{align}
    \E\left[(\vq_k^\T [\vutiltil_t]_k)^4\right] 
\leq  (3 \vee \gamma)^k \tn{\vq_k}^4, \quad \text{for all} \quad k = 1, \dots,L. \label{eq:forth_ind_tiltil}
\end{align}
We begin the proof, by showing that, $k=1$ obeys the induction as follows,
\begin{align}
	\E\left[(\vq_1^\T [\vutiltil_t]_1)^4\right] = \E \left[\left(q_{11}  + \vq_{12}^\T \vu_{t-1} \right)^4 \right]  &= q_{11}^4 + \E\left[(\vq_{12}^\T \vu_{t-1})^4\right] + 6 q_{11}^2 \E\left[(\vq_{12}^\T \vu_{t-1})^2\right], \nn \\\
    &\leqsym{a} q_{11}^4 + \gamma \tn{\vq_{12}}^4 + 6 q_{11}^2 \tn{\vq_{12}}^2, \nn \\
    &\leq (3 \vee \gamma) \left( q_{11}^4 + \tn{\vq_{12}}^4 + 2 q_{11}^2 \tn{\vq_{12}}^2\right), \nn \\
    &= (3 \vee \gamma) \left( q_{11}^2 + \tn{\vq_{12}}^2\right)^2, \nn \\
    &= (3 \vee \gamma) \tn{\vq_1}^4, \label{eqn:utiltil_fourth_induc_step1}
\end{align}
where (a) follows from Lemma~\ref{lemma:hyper_to_moment}. 
Suppose we have $\E\left[(\vq_{k-1}^\T [\vutiltil_t]_{k-1})^4\right] \leq  (3 \vee \gamma)^{k-1} \tn{\vq_{k-1}}^4$ for any $\vq_{k-1} \in \R^{(p+1)^{k-1}}$. Then, we apply the induction as follows,
\begin{align}
	\E\left[(\vq_k^\T [\vutiltil_t]_k)^4\right] 
     &= \E \left[ \left(\vq_{k}^\T ([\vutiltil_t]_{k-1} \otimes \ubb_{t-k})\right)^4\right], \nn \\
     &\eqsym{i} \E \left[\E \left[ \left( \ubb_{t-k}^\T\vQ_{k}[\vutiltil_t]_{k-1} \right)^4 \bgl \vu_{t-k}\right] \right], \nn \\
     &\leqsym{ii} (3 \vee \gamma)^{k-1}  \E \left[\tn{\vQ_k^\T\ubb_{t-k}}^4\right], \nn \\
     &\leqsym{iii} 
(3 \vee \gamma)^k  \tn{\vq_k}^4, \label{eqn:forth_ind_tiltil_proved}
\end{align}
 where we obtain (i) from setting $\vQ_k = \mat(\vq_k) \in \R^{(p+1) \times (p+1)^{k-1}}$, (ii) from the induction hypothesis, and (iii) is obtained from Lemma~\ref{lemma:hyper_to_moment} as follows: Let $\vq_{ki}^\T$ denote the $i$-th row of $\vQ_k^\T$. Then we have,
 
\begin{align}
    \E \left[\tn{\vQ_k^\T\ubb_{t-k}}^4\right] &= \E \left[ \left( \sum_{i=1}^{(p+1)^{k-1}} (\vq_{ki}^\T \ubb_{t-k})^2 \right)^2\right], \nn \\
    &= \E \left[  \sum_{i=1}^{(p+1)^{k-1}} (\vq_{ki}^\T \ubb_{t-k})^4 + \sum_{i=1}^{(p+1)^{k-1}} \sum_{\substack{j=1 \\ j \neq i}}^{(p+1)^{k-1}} (\vq_{ki}^\T \ubb_{t-k})^2 (\vq_{kj}^\T \ubb_{t-k})^2 \right], \nn \\
    &\leqsym{a} (3 \vee \gamma) \left( \sum_{i=1}^{(p+1)^{k-1}} \tn{\vq_{ki}}^4 + \sum_{i=1}^{(p+1)^{k-1}} \sum_{\substack{j=1 \\ j \neq i}}^{(p+1)^{k-1}} \tn{\vq_{ki}}^2\tn{\vq_{kj}}^2 \right), \nn \\
    &= (3 \vee \gamma) \left( \sum_{i=1}^{(p+1)^{k-1}} \tn{\vq_{ki}}^2\right)^2, \nn \\
    &= (3 \vee \gamma) \tf{\vQ_k}^4,
\end{align}
where we obtain (a) from \eqref{eqn:induction_step1_cross} and \eqref{eqn:utiltil_fourth_induc_step1}. Hence, we proved by induction that, for any $\vq_k \in \vR^{(p+1)^k}$, we have $\E\left[(\vq_k^\T [\vutiltil_t]_k)^4\right] \leq  (3 \vee \gamma)^k \tn{\vq_k}^4$, for all $k = 1, \dots,L$. 
Next, recalling the definition of $\vutil_t$ from \eqref{eqn:util_vec}, we use the intermediate result in \eqref{eq:forth_ind_tiltil} to get an upper bound on $\E\left[(\vv_k^\T [\vutil_t]_k)^4\right]$ as follows: 
For $k = 2, \dots, L$, we have,
\begin{align}
    \E\left[(\vv_k^\T [\vutil_t]_k)^4\right] &= \E \left[ \left(\vv_{k}^\T ([\vutiltil_t]_{k-1} \otimes \vu_{t-k})\right)^4\right], \nn \\
    &\eqsym{i} \E \left[\E \left[ \left( \vu_{t-k}^\T\vV_{k}[\vutiltil_t]_{k-1} \right)^4 \bgl \vu_{t-k}\right] \right], \nn \\
     &\leqsym{ii} (3 \vee \gamma)^{k-1}  \E \left[\tn{\vV_k^\T\vu_{t-k}}^4\right], \nn \\
     &\leqsym{iii} (3 \vee \gamma)^k  \tf{\vV_k}^4 = (3 \vee \gamma)^k  \tn{\vv_k}^4,
\end{align}
where we obtain (i) from setting $\vV_k = \mat(\vv_k) \in \R^{p \times (p+1)^{k-1}}$, (ii) is obtained from using \eqref{eq:forth_ind_tiltil}, and (iii) follows from Lemma~\ref{lemma:hyper_to_moment}. For $k=0,1$, we have, $\E\left[(\vv_0^\T [\vutil_t]_0)^4\right] = \E\left[(\vv_0^\T \vu_{t})^4\right] \leq \gamma \tn{\vv_0}^4$ and $\E\left[(\vv_1^\T [\vutil_t]_1)^4\right] = \E\left[(\vv_1^\T \vu_{t-1})^4\right] \leq \gamma \tn{\vv_1}^4$ using Lemma~\ref{lemma:hyper_to_moment}. Hence, we showed that,
\begin{align}
	\E\left[(\vv_k^\T [\vutil_t]_k)^4\right] 
\leq  (3 \vee \gamma)^{\max\{1,k\}} \tn{\vv_k}^4, \quad \text{for all} \quad k = 0, \dots,L. \label{eqn:forth_ind_til}
\end{align}

$\bullet$ {\bf Finalizing the Proof:} Putting it all together, for any $\vv \in \Sc^{d_{\util}-1}$ and $\vutil_t$ in \eqref{eqn:util_vec}, we have
\begin{align}
    \E \left[ \left( \vv^\T \vutil_t \right)^4\right] &= \E \left[ \left(\sum_{i=0}^L \vv_i^\T [\vutil_t]_i  \right)^4\right],  \nn \\
     &=  \sum_{i=0}^L \E \left[ \left(\vv_i^\T [\vutil_t]_i  \right)^4\right] + 3 \sum_{i=0}^L \sum_{\substack{j=0 \\ j \neq i}}^L \E \left[ \left(\vv_i^\T [\vutil_t]_i  \right)^2 \left(\vv_j^\T [\vutil_t]_j   \right)^2 \right], \nn \\
     &+ 6 \sum_{i=0}^L \sum_{\substack{j=0 \\ j \neq i}}^L\sum_{\substack{k=0 \\ k \neq i \\ k \neq j}}^L \E \left[ \left(\vv_i^\T [\vutil_t]_i  \right)^2 \left(\vv_j^\T [\vutil_t]_j \right) \left(\vv_k^\T [\vutil_t]_k \right) \right], \nn \\
    &\leq \sum_{i=0}^L (3 \vee \gamma)^{\max\{1,i\}} \tn{\vv_i}^4 + 3 \sum_{i=0}^L \sum_{\substack{j=0 \\ j \neq i}}^L (3 \vee \gamma)^{\min\{i,j\}} \tn{\vv_i}^2 \tn{\vv_j}^2 \nn \\
    &+ 6 \sum_{i=0}^L \sum_{\substack{j=0 \\ j \neq i}}^L\sum_{\substack{k=0 \\ k \neq i \\ k \neq j}}^L \sqrt{\E \left[ \left(\vv_i^\T [\vutil_t]_i  \right)^4\right] \E \left[\left(\vv_j^\T [\vutil_t]_j \right)^2 \left(\vv_k^\T [\vutil_t]_k \right)^2 \right]}, \nn \\
    &\leq (3 \vee \gamma)^L \bigg(\sum_{i=0}^L \tn{\vv_i}^4 + \sum_{i=0}^L \sum_{\substack{j=0 \\ j \neq i}}^L  \tn{\vv_i}^2 \tn{\vv_j}^2 \nn \\
    &+ 2 \sum_{i=0}^L \sum_{\substack{j=0 \\ j \neq i}}^L\sum_{\substack{k=0 \\ k \neq i \\ k \neq j}}^L \tn{\vv_i}^2 \tn{\vv_j}\tn{\vv_k}\bigg), \nn \\
    &\leq (3 \vee \gamma)^L \bigg(\sum_{i=0}^L \tn{\vv_i}^4 + \sum_{i=0}^L \sum_{\substack{j=0 \\ j \neq i}}^L  \tn{\vv_i}^2 \tn{\vv_j}^2 +  \sum_{\substack{j=0 \\ j \neq k}}^L\sum_{\substack{k=0 \\ k \neq j}}^L \left(\tn{\vv_j}^2 +\tn{\vv_k}^2\right)\bigg), \nn \\
    &\leq (3 \vee \gamma)^L  \left(1 + 2L\right) \leq  L(3 \vee \gamma)^{L+1}.
\end{align}
This completes the proof of Lemma~\ref{lemma:hyper_to_moment_cov}.
\end{proof}

\subsection{Proof of Theorem~\ref{thm:persistence}}
We are now ready to state the proof of our main result on persistence of excitation stated by Theorem~\ref{thm:persistence}. The proof follows a similar line of reasoning as that of Proposition 6.5 of \cite{sattar2024learning}. For the sake of completeness, we present the entire (modified) proof here.
\begin{proof}
To begin, recall the definition of $\vutil_t$ from \eqref{eqn:util_vec}, and let $\vUtil$ has rows $\{\vutil_t^\T\}_{t=L}^T$. From the proof of Lemma~\ref{lemma:hyper_to_moment_cov}, we have 
\begin{align}
	\E[\vUtil^\T \vUtil] &= \sum_{t = L}^{T} \E[\vutil_t \vutil_t^\T] = \sum_{t = L}^{T} \Iden_{d_{\util}} = (T-L+1) \Iden_{d_{\util}}.
\end{align}
Next, letting $\vv \in \Scal^{d_{\util}-1}$, we consider the quantity $\vv^\T \vUtil^\T \vUtil \vv = \sum_{t = L}^{T} (\vv^\T\vutil_t)^2$, which can be viewed as a summation of the random process $\{(\vv^\T\vutil_t)^2\}_{t = L}^{T}$. In the following, we will derive a one-sided concentration bound for this random process.

\medskip
$\bullet$ {\bf Step 1) Blocking:} We begin by using blocking technique to get independent samples as follows,
\begin{align}
	\sum_{t = L}^{T} (\vv^\T\vutil_t)^2 = \sum_{k = 0}^{L} \sum_{\tau = 1}^{(T-L+1)/(L+1)} (\vv^\T\vutil_{\tau (L+1) + k-1})^2,
\end{align}
where we make the simplifying assumption that $T-L+1$ can be divided by $L+1$. (\emph{Note that, this goes without loss of generality, and we assume it for the sake of clarity. It can be easily avoided by noting that 
$$
\sum_{t=L}^{T} \vutil_t \vutil_t^\T \succeq \sum_{t=L}^{ (L+1)\lfloor \frac{T-L+1}{L+1} \rfloor + L-1} \vutil_t \vutil_t^\T, 
$$
where $\lfloor\cdot \rfloor$ denotes the floor operator. We can then analyze everything with $T_0 = (L+1)\lfloor \frac{T-L+1}{L+1} \rfloor + L-1$, and note that $ T-L \le  T_0 \le T$.})

\medskip
$\bullet$ {\bf Step 2) Bernstein's inequality for non-negative random variables:} From Lemma~\ref{lemma:hyper_to_moment_cov}, we have $\E[(\vv^\T\vutil_t)^2] = 1$ and $\E[(\vv^\T\vutil_t)^4] \leq L(3 \vee \gamma)^{L+1}$ for any $\vv \in \Sc^{d_{\util}}$. Hence, we can use one-sided Bernstein's inequality for non-negative random variables~\citep{wainwright2019high} to obtain a lower bound on the smallest eigenvalue of $\sum_{t=L}^{T} \vutil_t \vutil_t^\T$. Specifically, we have,

\begin{align}
	\P \bigg(\sum_{\tau = 1}^{(T-L+1)/(L+1)} \big((\vv^\T\vutil_{\tau (L+1) + k-1})^2 - \E[(\vv^\T\vutil_{\tau (L+1) + k-1})^2]\big)  &\leq - \frac{T-L+1}{(L+1)}\eta\bigg) \nn \\ 
    &\leq \exp\bigg(-\frac{(T-L+1)\eta^2}{2(L+1)\Xi}\bigg), \nn \\
	\implies \P \bigg(\sum_{\tau = 1}^{(T-L+1)/(L+1)} (\vv^\T\vutil_{\tau (L+1) + k -1})^2  \leq \frac{T-L+1}{L+1} (1-\eta)\bigg) &\leq \exp\bigg(-\frac{(T-L+1)\eta^2}{2(L+1)\Xi}\bigg), \nn
\end{align} 
where we set $\Xi :=  L(3 \vee \gamma)^{L+1}$ for notational convenience. Union bounding over $L+1$ such events, we get the following,
\begin{align}
	\P \bigg(\sum_{k = 0}^{L}\sum_{\tau = 1}^{(T-L+1)/(L+1)} (\vv^\T\vutil_{\tau (L+1) + k -1})^2  &\leq (T-L+1) (1-\eta)\bigg) \nn \\
    &\leq (L+1)\exp\bigg(-\frac{(T-L+1)\eta^2}{2(L+1)\Xi}\bigg). \label{eqn:one_sided_bernstein}
\end{align}
This can be alternately represented by letting
\begin{align}
	(L+1)\exp\bigg(-\frac{(T-L+1)\eta^2}{2(L+1)\Xi}\bigg) &= \delta, \nn \\
	\iff \frac{(T-L+1)\eta^2}{2(L+1)\Xi} &= \log((L+1)/\delta), \nn \\
	\impliedby \eta &= \sqrt{\frac{L+1}{T-L+1}2 \Xi\log((L+1)/\delta)}. 
\end{align}
Hence, we have
\begin{align}
	\P \bigg(\sum_{t = L}^{T} (\vv^\T\vutil_t)^2  \leq (T-L+1) -\sqrt{2\Xi(L+1)(T-L+1)\log((L+1)/\delta)}\bigg) &\leq \delta.
\end{align}

$\bullet$ {\bf Step 3) Covering with ${\delta}/{(8d_{\util})}$-net:} Next, we use a covering argument as follows: Let $\Ncal_\epsilon := \{\vv_1, \vv_2, \dots, \vv_{|\Ncal_\epsilon|} \} \subset \Scal^{d_{\util}-1}$ be the $\epsilon$-net of $\Scal^{d_{\util}-1}$ such that for any $\vv \in \Scal^{d_{\util}-1}$, there exists $\vv_i \in \Ncal_\epsilon$ such that $\tn{\vv - \vv_i} \leq \epsilon$. From Lemma 5.2 of \cite{vershynin2010introduction}, we have $|\Ncal_\epsilon| \leq (1 + 2/ \epsilon)^{d_{\util}}$. 

Let us choose $\vv \in \Scal^{d_{\util}-1}$ for which $\lambda_{\min} (\vUtil^\T \vUtil) = \vv^\T \vUtil^\T \vUtil \vv$, and choose $\vv_i \in \Ncal_\epsilon$ which approximates $\vv$ as $\tn{\vv - \vv_i} \leq \epsilon$. By triangle inequality, we have
\begin{align}
	|\vv^\T \vUtil^\T \vUtil \vv  - \vv_i^\T \vUtil^\T \vUtil \vv_i| &= |\vv^\T \vUtil^\T \vUtil ( \vv - \vv_i) + (\vv - \vv_i)^\T \vUtil^\T \vUtil \vv_i|, \nn \\
	& \leq \norm{\vUtil^\T \vUtil}_\op \tn{\vv} \tn{\vv - \vv_i} + \norm{\vUtil^\T \vUtil}_\op \tn{\vv_i} \tn{\vv - \vv_i}, \nn \\
	& \leq 2 \epsilon \norm{\vUtil^\T \vUtil}_\op, \nn \\
    \implies  \vv^\T \vUtil^\T \vUtil \vv &\geq  \vv_i^\T \vUtil^\T \vUtil \vv_i - 2 \epsilon \norm{\vUtil^\T \vUtil}_\op, \nn \\
     \implies \lambda_{\min}(\vUtil^\T \vUtil ) &\geq \inf_{\vv_i \in \Ncal_\epsilon} \vv_i^\T \vUtil^\T \vUtil \vv_i - 2 \epsilon \norm{\vUtil^\T \vUtil}_\op.
\end{align}
Hence, in order to lower bound $\lambda_{\min}(\vUtil^\T \vUtil )$, we also need an upper bound on $\norm{\vUtil^\T \vUtil}_\op$. This can be done as follows: First, we have 
\begin{align}
	\E[\norm{\vUtil^\T \vUtil}_\op] = \E[\lambda_{\max}(\vUtil^\T \vUtil)] &\leq \E[\tr(\vUtil^\T\vUtil)]= \tr\big(\E[\vUtil^\T \vUtil]\big), \nn \\
    & = (T-L+1)\tr\big(\Iden_{d_{\util}}\big) = d_{\util}(T-L+1).
\end{align}
Hence, using Markov's inequality, we get
\begin{align}
	\P\bigg( \norm{\vUtil^\T \vUtil}_\op > \frac{d_{\util}(T-L+1)}{\delta}\bigg) \leq  \frac{\E[ \norm{\vUtil^\T \vUtil}_\op ]}{d_{\util}(T-L+1)} \delta \leq \delta.
\end{align}
Let  $\Ecal := \{\norm{\vUtil^\T \vUtil}_\op  \leq \frac{2d_{\util}(T-L+1)}{\delta}\}$ denote the event that $\norm{\vUtil^\T \vUtil}_\op$ is upper bounded by a specified threshold. Then, it is straightforward to see that $\P(\Ecal) \geq 1 - \delta/2$. This further implies,
\begin{align}
	\P\bigg(\lambda_{\min}(\vUtil^\T\vUtil) &< (T-L+1)(1/2-\eta)\bigg) \nn \\
    &\leq \P\bigg(\big\{\lambda_{\min}(\vUtil^\T\vUtil) < (T-L+1)(1/2-\eta)\big\} \bigcap \Ecal \bigg) + \P (\Ecal^c), \nn \\
	&\leq \P\bigg(\big\{\inf_{\vv_i \in \Ncal_\epsilon} \vv_i^\T \vUtil^\T \vUtil \vv_i - 2 \epsilon \norm{\vUtil^\T \vUtil}_\op < (T-L+1)(1/2-\eta)\big\} \bigcap \Ecal \bigg) + \delta/2, \nn \\
	&\leq \P\bigg(\inf_{\vv_i \in \Ncal_\epsilon} \vv_i^\T \vUtil^\T \vUtil \vv_i  < (T-L+1)(1/2-\eta) +  4 \epsilon\frac{d_{\util}(T-L+1)}{\delta}\bigg) + \delta/2, \nn \\
	&= \P\bigg(\inf_{\vv_i \in \Ncal_\epsilon} \vv_i^\T \vUtil^\T \vUtil \vv_i  < (T-L+1)(1/2-\eta +  \frac{4 \epsilon d_{\util}}{\delta})\bigg) + \delta/2, \nn \\
	&= \P\bigg(\inf_{\vv_i \in \Ncal_\epsilon} \vv_i^\T \vUtil^\T \vUtil \vv_i  < (T-L+1)(1-\eta )\bigg) + \delta/2, \label{eqn:before_union_bound}
\end{align}
where we obtained the last inequality by choosing $\epsilon = \frac{\delta}{8d_{\util}}$. Using~\eqref{eqn:one_sided_bernstein} with union bounding over all the elements in $\Ncal_\epsilon$, we obtain,
\begin{align}
	 \P\bigg(\inf_{\vv_i \in \Ncal_\epsilon} \vv_i^\T \vUtil^\T \vUtil \vv_i  < (T-L+1)(1-\eta )\bigg) &\leq |\Ncal_\epsilon| (L+1)\exp\bigg(-\frac{(T-L+1)\eta^2}{2(L+1)\Xi}\bigg).
\end{align}
 From Lemma 5.2 of \cite{vershynin2010introduction}, we have $|\Ncal_\epsilon| \leq (1 + 2/ \epsilon)^{d_{\util}}$. Hence, we have
\begin{align}
    \P\bigg(\inf_{\vv_i \in \Ncal_\epsilon} \vv_i^\T \vUtil^\T \vUtil \vv_i  < (T-L+1)(1-\eta )\bigg) &\leq (L+1)(1 + \frac{16 d_{\util}}{\delta})^{d_{\util}}\exp\bigg(-\frac{(T-L+1)\eta^2}{2(L+1)\Xi}\bigg). \nn
\end{align}
This can be alternately represented by letting,
\begin{align}
	(L+1)(1 + \frac{16 d_{\util}}{\delta})^{d_{\util}} \exp\bigg(-\frac{(T-L+1)\eta^2}{2(L+1)\Xi}\bigg) = \delta/2, \nn \\
	 \iff \exp\bigg(-\frac{(T-L+1)\eta^2}{2(L+1)\Xi}\bigg) = \delta/(2(L+1)) (1 + \frac{16 d_{\util}}{\delta})^{-d_{\util}}, \nn \\
\iff \frac{(T-L+1)\eta^2}{2(L+1)\Xi} = \log(2(L+1)/\delta) +d_{\util} \log\big(1 + \frac{16 d_{\util}}{\delta}\big), \nn \\
	\impliedby \eta = \sqrt{\frac{2(L+1)\Xi}{(T-L+1)}\big(\log\big(\frac{2(L+1)}{\delta}\big) +d_{\util} \log\big(1 + \frac{16 d_{\util}}{\delta}\big)\big)}.
\end{align}
Plugging this back into \eqref{eqn:before_union_bound}, we have
\begin{align}
		\P\bigg(\lambda_{\min}(\vUtil^\T\vUtil) &< (T-L+1)/2 \nn \\
        &-   \sqrt{2(L+1)(T-L+1)\Xi\big(\log\big(\frac{2(L+1)}{\delta}\big) +d_{\util} \log\big(1 + \frac{16 d_{\util}}{\delta}\big)\big)}\bigg) \leq \delta.  \nn 
\end{align}
Finally, choosing the trajectory length via
\begin{align}
	(T-L+1)/4 &\geq \sqrt{2(L+1)(T-L+1)\Xi\big(\log\big(\frac{2(L+1)}{\delta}\big) + d_{\util} \log\big(1 + \frac{16 d_{\util}}{\delta}\big)\big)}, \nn \\
	 \iff (T-L+1)/16 &\geq 2(L+1)\Xi\big(\log\big(\frac{2(L+1)}{\delta}\big) +d_{\util} \log\big(1 + \frac{16 d_{\util}}{\delta}\big)\big), \nn \\
	 \impliedby T - L+1 &\geq 32(L+1)\Xi\big(\log\big(\frac{2(L+1)}{\delta}\big) +d_{\util} \log\big(1 + \frac{16 d_{\util}}{\delta}\big)\big),
\end{align}
we obtain the following persistence of excitation result, 
\begin{align}
	\P\big(\lambda_{\min}(\vUtil^\T\vUtil) \geq (T-L+1)/4 \big) \geq 1 - \delta.  
\end{align}
This completes the proof of Theorem~\ref{thm:persistence}.
\end{proof}

\begin{remark}[Covering with $1/(4 d_{\util})$-net:]
    Note that Theorem~\ref{thm:persistence} holds for a broader class of inputs satisfying $(4,2,\gamma)$-hypercontractivity condition. For specific choice of inputs, we can leverage the additional information to get a tighter sample complexity lower bound. For example, if the inputs are sampled independently according to $\mathrm{Unif}(\sqrt{p} \cdot \cS^{p-1})$, then we can use a different covering argument in the proof of Theorem~\ref{thm:persistence}. Specifically, observing that
    \begin{align}
        \norm{\vUtil^\T \vUtil}_\op  \leq  \sum_{t = L}^{T} \norm{ \vutil_t \vutil_t^\T]}_\op &\leq (T-L+1) \max_{t \in [L,T]} \tn{\vutil_t}^2, \nn \\
        &= (T-L+1) \left(p+\sum_{i = 0}^{L-1} p(p+1)^i\right) =  (T-L+1)d_{\util}, 
\end{align}
    we repeat the Step 3) in the proof of Theorem~\ref{thm:persistence} with $\Ecal := \{\norm{\vUtil^\T \vUtil}_\op  \leq (T-L+1)d_{\util}\}$ and using $\epsilon = \frac{1}{4 d_{\util}}$-net covering argument to get the following persistence of excitation result: For all $\delta \in (0,1)$, the event, 
\begin{align}
    \lambda_{\min}\left(\vUtil^\T \vUtil\right)  \geq (T-L+1)/4.
\end{align}
holds with probability at least $1- \delta$, provided that 
\begin{align}
    T - L+1 \geq 32(L+1)\Xi\big(\log\big({(L+1)}/{\delta}\big) + d_{\util} \log\big(1 + 8 d_{\util}\big)\big).
\end{align}
\end{remark} \section{Proof of Theorem \ref{thm:main}}

In this appendix, we provide the proofs of Proposition  \ref{prop:truncation bias} and Proposition \ref{prop:multiplier process}. These propositions together with persistence of excitation presented in Theorem \ref{thm:persistence} lead immediately to Theorem \ref{thm:main} which concerns the main guarantee on the estimation of the \emph{Markov-like} parameters. We present its proof in this appendix too.

\subsection{Proof of Theorem \ref{thm:main}}\label{app:proof main}

The proof is an immediate consequence of Theorem \ref{thm:persistence}, Proposition \ref{prop:truncation bias}, and Proposition \ref{prop:multiplier process}. Indeed, define the events 
\begin{align*}
    \cE_1 & \triangleq \left\lbrace  \left \Vert \sum_{t=L}^T \vutil_t  \vepsilon_t^\top   \right\Vert_\op \!\!\!\!\!\!  \le \frac{\Vert \vC \Vert_\op (4p\Vert \vB \Vert_\op^2 + \sigma^2) \kappa^2  \rho^{L-1}  }{1-\rho} \sqrt{2 (T-L) (p+1)^{L+1}\log\left(\frac{2 \cdot9^{d_{\vutil} + m }}{\delta}\right)} \right\rbrace \\
    \cE_2 & \triangleq \left\lbrace  \left\Vert \sum_{t = L}^{T} \vutil_t(\vF\vwtil_t)^\T\right\Vert_\op  \!  \! 
 \!  \! \! \! \le    \frac{2(1+\kappa \Vert \vC \Vert_\op)\sigma }{1- \rho}    \sqrt{2 L (T-L+2)(p+1)^{L+1} \log\left(\frac{2L9^{d_{\vutil} + m+nL}}{\delta}\right)}  \right\rbrace \\
    \cE_3 & \triangleq \left\lbrace  \lambda_{\min}\left( \sum_{t=L}^{T} \vutil_t \vutil_t^\T \right) \geq (T-L+1)/4. \right\rbrace 
\end{align*}
Recalling the estimation error decomposition \eqref{eq:error decomposition}, we see that when the event $\cE_1 \cap \cE_2 \cap \cE_3$ holds, then the upper bound on the estimation error presented in Theorem \ref{thm:main} follows. Now, we remark that by union bound, we have 
\begin{align*}
    \PP(\cE_1 \cap \cE_2 \cap \cE_3) = 1 - \PP(\cE_1^c \cup \cE_2^c \cup \cE_3^c) \ge 1 -  \PP(\cE_1^c) - \PP ( \cE_2^c) - \PP(\cE_3^c). 
\end{align*}
Thus, using Theorem \ref{thm:persistence}, Proposition \ref{prop:truncation bias}, and Proposition \ref{prop:multiplier process}, we obtain  $ \PP(\cE_1 \cap \cE_2 \cap \cE_3) \le 3 \delta$, provided the condition \eqref{eqn:trajectory_size_main} in Theorem \ref{thm:persistence} holds. This concludes the proof.

\subsection{Proof of Proposition \ref{prop:multiplier process}}\label{app:multiplier process}
    Using the submultiplicativity of the norm and Lemma \ref{lem:upper bound stability}, we have:
    \begin{align}\label{eq:upper bound T3}
        \left\Vert \sum_{t=L}^T \vutil_t (\vF\tilde{\vw}_t)^\T \right\Vert_\op \le  \left\Vert \sum_{t=L}^T \vutil_t \tilde{\vw}_t^\T \right\Vert_\op  \left\Vert \vF  \right\Vert_\op \le \left( 1+ \frac
        {\kappa \Vert \vC \Vert_\op}{1- \rho}\right) \left\Vert \sum_{t=L}^T \vutil_t \tilde{\vw}_t^\T \right\Vert_\op 
    \end{align}
    We will now focus on bounding $\left\Vert \sum_{t=L}^T \vutil_t \tilde{\vw}_t^\T \right\Vert_\op$ with high probability. Recalling the variational form of the operator norm, we using a $1/4$-net argument. Let $\cM$ (resp. $\cN$) be an $1/4$-net of $\cS^{d_{\vutil} - 1}$ (resp. $\cS^{m + nL - 1}$) with minimal cardinality. By Lemma \ref{lem:net argument}, we obtain: for all $u > 0$ 
    \begin{align}\label{eq:net argument for T3}
       \PP\left( \left\Vert \sum_{t=L}^T \vutil_t \tilde{\vw}_t^\T \right\Vert_\op > 2 u\right) \le 9^{d_{\vutil} + m + nl} \max_{\theta \in \cM, \lambda \in \cN}\PP\left(  \sum_{t = L}^{T} (\theta^\T\vutil_t) (\lambda^\T \vwtil_t)   > u\right)
    \end{align}
We observe that the sequences $\lbrace \vwtil_{t} \rbrace_{ t\ge L}$ and $\lbrace \vutil_t \rbrace_{t \ge L}$ are independent, but these are not sequences of independent random vectors. We will use a blocking trick to handle this. Let $\theta \in \cS^{d_{\vutil}-1}, \lambda \in \cS^{m + nL - 1}$. We have: 
    \begin{align}
        \left\vert \sum_{t = L}^{T} (\theta^\T\vutil_t) (\lambda^\T \vwtil_t)\right\vert  & =  \left\vert \sum_{\ell = 1}^L \sum_{s \in \cT_\ell} (\theta^\T\vutil_s) (\lambda^\T \vwtil_s) \right\vert \le  \sum_{\ell = 1}^L \left \vert \sum_{s \in \cT_\ell} (\theta^\T\vutil_s) (\lambda^\T \vwtil_s) \right\vert
    \end{align}
    where for all $\ell$, $\cT_\ell$ correspond to the indices that satisfy $(t \mod L) = \ell - 1$ and we note that $ \lfloor (T-L +1)/L \rfloor \le \vert \cT_\ell \vert \le \lceil (T-L +1)/L \rceil$. We note that $\lbrace \vwtil_s \rbrace_{s \in \cT_\ell}$ are independent for all $\ell \in [L]$.  We use Lemma \ref{lem:freedman} to bound with high probability the sum $\sum_{s \in \cT_\ell} (\theta^\T\vutil_s) (\lambda^\T \vwtil_s)$ for each partition $\cT_\ell$, then combine these bounds with a union bound over the $L$ partitions to conclude:  for all $\delta \in (0,1)$, 
    \begin{align}
         \PP\left(\left\vert \sum_{t = L}^{T} (\theta^\T\vutil_t) (\lambda^\T \vwtil_t)\right\vert  >  \sigma \sqrt{2 L (T-L+2)(p+1)^{L+1} \log\left(\frac{2L9^{d_{\vutil} + m+nL}}{\delta}\right)}\right) \le \frac{\delta}{9^{d_{\vutil} + m+nL}}.
    \end{align}
    Combining this bound with the inequality \eqref{eq:net argument for T3} yields 
    \begin{align}
         \PP\left(\left\Vert \sum_{t = L}^{T} \vutil_t\vwtil_t^\T\right\Vert_\op  \le 2\sigma  \sqrt{2 L (T-L+2) (p+1)^{L+1}\log\left(\frac{2L9^{d_{\vutil} + m+nL}}{\delta}\right)}\right) \ge 1 -\delta.
    \end{align}
    Recalling the inequality  
    \eqref{eq:upper bound T3}, we conclude that: for all $\delta \in (0,1)$, the event 
    \begin{align}
        \left\Vert \sum_{t = L}^{T} \vutil_t(\vF\vwtil_t)^\T\right\Vert_\op  \le 2  \left(1 + \frac{\kappa \Vert \vC \Vert_\op}{1- \rho} \right) \sigma  \sqrt{2 L (T-L+2)(p+1)^{L+1} \log\left(\frac{2L9^{d_{\vutil} + m+nL}}{\delta}\right)}
    \end{align}
    holds with probability $1- \delta$.

\subsection{Proof of Proposition \ref{prop:truncation bias}}\label{app:truncation bias}

First, using the variational form of the operator norm, we see that:
\begin{align*}
   \left \Vert \sum_{t=L}^T \vutil_t  \vepsilon_t^\top   \right\Vert_\op  = \sup_{\theta \in \cS^{d_{\vutil} - 1}, \lambda \in \cS^{m - 1}}   \sum_{t=L}^T (\theta^\top\vutil_t) (\lambda^\top \vepsilon_t). 
\end{align*}
We use an $1/4$-net argument to bound the supremum. Let $\cM$ (resp. $\cN$) be $1/4$-nets with minimal cardinality of the sphere $\cS^{d_{\vutil} - 1}$ (resp. $\cS^{m - 1}$). Thus, using Lemma \ref{lem:net argument}   we obtain: for all $r > 0$: 
\begin{align}\label{eqn:eq net argument}
    \PP\left(    \left \Vert \sum_{t=L}^T \vutil_t  \vepsilon_t^\top   \right\Vert_\op  > 2 r \right) \le 9^{d_{\vutil}+m}\max_{\theta \in \cM, \lambda \in \cN} \PP\left(  \sum_{t=L}^T (\theta^\top\vutil_t) (\lambda^\top \vepsilon_t) > r \right).
\end{align}
It remains to bound $\sum_{t=L}^T (\theta^\top\vutil_t) (\lambda^\top \vepsilon_t)$ with high probability uniformly over the unit spheres. Let $\theta \in \cS^{d_{\vutil} - 1}$, and $\lambda \in \cS^{m - 1}$. We have: 
    \begin{align*}
        & \sum_{t=L}^T (\theta^\top\vutil_t) (\lambda^\top \vepsilon_t)  =  \sum_{t=L}^T (\theta^\top\vutil_t) (\lambda^\top  \Cb) \left( \prod_{\ell = 1}^{L-1} (\ub_{t-\ell}\circ \Ab )  \right)  \vx_{t-L} \\
        & \qquad = \sum_{t=0}^{T-L} (\theta^\top\vutil_{t+L}) (\lambda^\top  \Cb) \left( \prod_{\ell = 1}^{L-1} (\ub_{t + \ell}\circ \Ab )  \right)  \vx_{t} \\
& \qquad \overset{(a)}{=} \sum_{t=1}^{T-L} (\theta^\top\vutil_{t+L}) (\lambda^\top  \Cb) \left( \prod_{\ell = 1}^{L-1} (\ub_{t + \ell}\circ \Ab )  \right)  \sum_{s = 0}^{t-1} \left(\prod_{k = s + 1}^{t-1} (\vu_{k} \circ \vA)\right) (\vB \vu_{s} + \vw_{s}) \\
         & \qquad = \sum_{t=1}^{T-L} \sum_{s = 0}^{t-1}  (\theta^\top\vutil_{t+L}) (\lambda^\top  \Cb) \left( \prod_{\ell = 1}^{L-1} (\ub_{t + \ell}\circ \Ab )  \right)  \left(\prod_{k = s + 1}^{t-1} (\vu_{k} \circ \vA)\right) (\vB \vu_{s} + \vw_{s}), 
         \\ & \qquad \overset{(b)}{=} \sum_{t=0}^{T-L-1} \sum_{s = 0}^{t}  \underbrace{(\theta^\top\vutil_{t+1+L}) (\lambda^\top  \Cb) \left( \prod_{\ell = 1}^{L-1} (\ub_{t + 1+  \ell}\circ \Ab )  \right)}_{:= \vM(\vu_{t+2}, \dots, \vu_{t+1+L})}  \underbrace{\left(\prod_{k = s + 1}^{t} (\vu_{k} \circ \vA)\right)}_{:=\vN(\vu_{s+1}, \dots, \vu_{t})} (\vB \vu_{s} + \vw_{s}) 
    \end{align*}
    where we used in $(a)$, the dynamics \eqref{eqn:bilinear sys} to express $\vx_t$ in terms of $(\vu_{s}, \vw_s)_{s < t}$ (e.g., see \eqref{eqn:bilinear sys state}). We also introduce in $(b)$ the quantities $\vM$ and $\vN$ which depends on inputs. Next, we perform next a change of indices to obtain: 
   \begin{align}\label{eq:martingale bias}
        \sum_{t=L}^T (\theta^\top\vutil_t) (\lambda^\top \vepsilon_t)  & = \sum_{s=0}^{T-L-1} \underbrace{\left( \sum_{t = s}^{T-L-1} \vM(\vu_{t+2}, \dots, \vu_{t+L+1}) \vN (\vu_{s+1}, \dots, \vu_t) \right)}_{:= f_s(\vu_{s+1}, \dots, \vu_{T})} (\vB \vu_{s} + \vw_{s}) \nonumber \\
        & \overset{(c)}{=} \sum_{s=0}^{T-L-1} f_s(\vu_{s+1}, \dots, \vu_{T}) (\vB \vu_{s} + \vw_{s}),
    \end{align}
    where we introduce in $(c)$ the functions $f_s(\cdot)$. Now we clearly see that $\sum_{t=L}^T (\theta^\top\vutil_t) (\lambda^\top \vepsilon_t)$, written in the form \eqref{eq:martingale bias} is a martingale difference. Before we use this fact, let us note that using Lemma \ref{lem:upper bound stability} we can show that the terms involving the functions $f_s(\cdot)$ are well bounded. More specifically, we have 
\begin{align*}
        \Vert f(\vu_{s+1}, \dots, \vu_{T}) \Vert_{\ell_2} & =  \left\Vert \sum_{t = s}^{T-L-1} \vM(\vu_{t+1}, \dots, \vu_{t+L+1}) \vN (\vu_{s+1}, \dots, \vu_t)  \right\Vert_{\ell_2} \\
        & \le  \sum_{t = s}^{T-L-1} \left\Vert \vutil_{t+L+1} \right\Vert_{\ell_2} \left \Vert \vC\right\Vert_\op \left\Vert \prod_{\ell = 1}^{L-1} (\ub_{t + 1 + \ell}\circ \Ab ) \right\Vert_{\op} \left\Vert \prod_{k = s+1}^{t} (\ub_{ k}\circ \Ab ) \right\Vert_{\op}  \\
        & \le \sup_{t \le T}\Vert \vutil_t\Vert_{\ell_2}\Vert \vC \Vert_\op \kappa^2  \rho^{L-1} \sum_{t = 0}^{T-L-s-1}  \rho^{t} \\
        &  \le  \frac{\sqrt{(p+1)^{L+1}}\Vert \vC \Vert_\op \kappa^2\rho^{L-1}}{1-\rho}
    \end{align*}
    where we bound $\Vert \vutil_t \Vert_{\ell_2} \le \sqrt{d_{\util}} \leq \sqrt{(p+1)^{L+1}}$. Now, we also remark that $\vB \vu_s + \vw_s$ is zero-mean and $(4 p\Vert \vB\Vert_\op^2 + \sigma^2)$-subgaussian. Hence, using Freedman's inequality (see Lemma \ref{lem:freedman}), we obtain: for all $\delta \in (0,1)$, the event 
    
    \begin{align*}
        & \left\vert \sum_{s=0}^{T-L-1}f(\vu_{s+1}, \dots, \vu_{T}) (\vB \vu_{s} + \vw_{s}) \right\vert \\
        & \qquad \qquad \qquad \qquad >  \frac{\Vert \vC \Vert_\op (4p\Vert \vB \Vert_\op^2 + \sigma^2) \kappa^2  \rho^{L-1} \sqrt{2 (T-L) (p+1)^{L+1}\log(2 \cdot9^{d_{\vutil} + m }/\delta)} }{1-\rho} 
    \end{align*}
    with probability at most $\delta/9^{d_{\vutil} + m }$. The conclusion follows immediately by recalling the inequality \eqref{eqn:eq net argument}.

 \section{Miscellaneous Lemmas \& Concentration Tools}

In this Appendix we present a set of lemmas and concentration inequalities that we persistently make use of in our proofs. First, we provide a version of Freedman's inequality which can also be deduced from Azuma-Hoeffding's inequality.  
\begin{lemma}\label{lem:freedman}
    Let $(\cF_t)_{t\ge 0}$ be a filtration. Let $(\boldsymbol{\eta}_t)_{t \ge 1}$ is a sequence of zero-mean, $\sigma^2$-subgaussian random vectors taking values in $\R^{d}$, such that $\veta_t$ is $\cF_t$-measurable for all $t\ge 1$. Let $(\vx_t)$ be a sequence of random vectors taking values in $\R^{d}$ such that for all $t \ge 1$, $\vx_t$ is $\cF_{t-1}$-measurable and $\Vert \vx_t \Vert_{\ell_2} \le K$ almost surely for some $K > 0$. Then for all $\delta \in (0,1)$, $T \ge 1$, 
    \begin{align*}
        \PP\left( \left\vert \sum_{t=1}^T \vx_t^\T \veta_t \right\vert \le  \sigma K \sqrt{2 T \log(2/\delta)}\right) \ge 1 - \delta
    \end{align*}
\end{lemma}
Next, we present an immediate generalization of Hoeffding's lemma for bounded random vectors.   
\begin{lemma}
    Let $\vu$ be a $p$-dimensional random vector sampled from $ \mathrm{Unif}(\sqrt{p} \cdot \cS^{p-1})$. Then, it holds that  $\vu$ is zero-mean and $4p$-subgaussian, i.e., $\EE[\exp(\theta^\top \vu)] \le \exp(2 \Vert \theta \Vert^2_{\ell_2} p  )$, for all $\theta \in \RR^{p}$. 
\end{lemma}
Finally, we formalize the trick of a net arguments in the following lemma which is a classical argument that can be found in \cite{vershynin2010introduction}: 
\begin{lemma}[$\epsilon$-net argument]\label{lem:net argument}
    Let $\vW$ be a  $m\times n$ random matrix and $\varepsilon \in (0,1/2)$. Let $\cM$  (resp. $\cN$) be an $\varepsilon$-net of $(\cS^{m - 1}, \Vert \cdot \Vert_{\ell_2})$ (resp. $(\cS^{n - 1}, \Vert \cdot \Vert_{\ell_2})$). For all $\rho>0$, it holds that  
    \begin{align*}
        \PP\left( \Vert \vW \Vert_\op  > \frac{\rho}{1-2\varepsilon}\right) \le \left(1 + \frac{2}{\varepsilon}\right)^{n+m}\max_{\vx \in \cM, \vy \in \cN} \PP( \vx^\T \vW \vy > \rho).
    \end{align*}
\end{lemma}
We omit the proofs of these Lemmas as they are standard results within the literature.

\end{document}